\documentclass[10pt]{article} 

\usepackage{etoolbox}
\newcommand{\arxiv}[1]{\iftoggle{iclr}{}{#1}}
\newcommand{\iclr}[1]{\iftoggle{iclr}{#1}{}}
\newtoggle{iclr}
\global\toggletrue{iclr}
\global\togglefalse{iclr}
\usepackage{makecell}
\iclr{
\PassOptionsToPackage{dvipsnames}{xcolor}
\usepackage{icml2026}
}

\usepackage[utf8]{inputenc} 
\usepackage[T1]{fontenc}    
\usepackage{url}            
\usepackage{booktabs}       
\usepackage{amsfonts}       
\usepackage{nicefrac}       
\usepackage{microtype}      
\usepackage{bm}
\usepackage{fontawesome5}
\usepackage{wrapfig}
\usepackage{mdframed}
\usepackage{subcaption}
\usepackage{multirow}

\newcommand{\bigO}{\mathcal{O}}

\usepackage{enumitem}
\iclr{\setlist[itemize]{noitemsep, topsep=0pt}}
\iclr{\setlist[enumerate]{noitemsep, topsep=0pt}}
\usepackage{breakcites}

\newtoggle{draft}
\togglefalse{draft}

\usepackage{mathrsfs}

\usepackage{verbatim}

\usepackage{multicol}
\usepackage{siunitx}
\usepackage{colortbl}
\usepackage{xcolor}
\definecolor{hl}{RGB}{245,245,245}
\newcolumntype{H}{>{\columncolor{hl}}} 

\definecolor{sd}{RGB}{232, 233, 251}
\definecolor{fm}{RGB}{252, 232, 216}

\usepackage{setspace}

\usepackage{transparent}

\usepackage{inconsolata}
\usepackage[scaled=.90]{helvet}
\usepackage{xspace}
\usepackage[most]{tcolorbox}
\tcbuselibrary{breakable}

\usepackage{pifont}
\newcommand{\xmark}{\ding{55}}%

\RequirePackage{algorithm}
\RequirePackage{algorithmic}

\arxiv{
\usepackage[a4paper, margin=2.5cm]{geometry}
\PassOptionsToPackage{hypertexnames=false}{hyperref}  
\usepackage[dvipsnames]{xcolor}
\usepackage[colorlinks]{hyperref}
\hypersetup{
    citecolor=[RGB]{50,100,170},
    linkcolor=[RGB]{50,100,170},
    urlcolor=[RGB]{255,102,178}}
\usepackage{fancyhdr}
}

\arxiv{
\newcommand{\toptitlebar}{%
  {\color{black}\hrule height 1pt}%
  \vskip 0.25in%
}
}

\makeatletter
\arxiv{
\renewcommand{\maketitle}{%
  \begin{center}%
    \toptitlebar
    \vskip 0.1in%
    {\LARGE\bfseries \@title \par}%
    \vskip 0.3in%
    {\normalsize \@author \par}%
  \end{center}%
  \par
  \vskip 0.3in%
}

\renewcommand\section{\@startsection {section}{1}{\z@}{-2.0ex plus
    -0.5ex minus -.2ex}{1.5ex plus 0.3ex minus .2ex}{\large\bfseries\raggedright}}
\renewcommand\subsection{\@startsection{subsection}{2}{\z@}{-1.8ex plus
    -0.5ex minus -.2ex}{0.8ex plus .2ex}{\normalsize\bfseries\raggedright}}
\renewcommand\subsubsection{\@startsection{subsubsection}{3}{\z@}{-1.5ex plus
   -0.5ex minus -.2ex}{0.5ex plus .2ex}{\normalsize\bfseries\raggedright}}

\renewenvironment{abstract}%
  {\centerline{\large\bfseries Abstract}%
   \begin{list}{}%
      {\setlength{\rightmargin}{0.6cm}%
       \setlength{\leftmargin}{0.6cm}}%
    \item[]\ignorespaces}%
  {\unskip\end{list}}

\setlength{\parindent}{2em}
\setlength{\parskip}{0.4em}
}
\makeatother

\iclr{
\usepackage[colorlinks=true, linkcolor=blue!70!black, citecolor=blue!70!black,urlcolor=blue!70!black,breaklinks=true]{hyperref}
\PassOptionsToPackage{dvipsnames}{xcolor} 
}

\usepackage{microtype}
\usepackage{hhline}

\makeatletter
\newcommand{\neutralize}[1]{\expandafter\let\csname c@#1\endcsname\count@}
\makeatother

\usepackage{algorithm}

\arxiv{
\usepackage{natbib}
\bibliographystyle{plainnat}
\bibpunct{(}{)}{;}{a}{,}{,}
}

\usepackage{amsthm}
\usepackage{mathtools}
\usepackage{amsmath}
\usepackage{bbm}
\usepackage{amsfonts}
\usepackage{amssymb}

\usepackage{xpatch}

\usepackage{thmtools}
\usepackage{thm-restate}
\declaretheorem[name=Theorem]{theorem}
\declaretheorem[name=Lemma,sibling=theorem]{lemma}
\declaretheorem[name=Assumption,sibling=theorem]{assumption}
\declaretheorem[name=Condition,sibling=theorem]{condition}

\declaretheorem[name=Proposition,sibling=theorem]{proposition}

\makeatletter
  \renewenvironment{proof}[1][Proof]%
  {%
   \par\noindent{\bfseries\upshape {#1.}\ }%
  }%
  {\qed\newline}
  \makeatother

\theoremstyle{definition}  

\newtheorem{corollary}{Corollary}[section]

\theoremstyle{plain}
\newtheorem{definition}{Definition}[section]

\xpatchcmd{\proof}{\itshape}{\normalfont\proofnameformat}{}{}
\newcommand{\proofnameformat}{\bfseries}

\usepackage[nameinlink,capitalize]{cleveref}

\renewcommand{\eqref}[1]{\texorpdfstring{\hyperref[#1]{(\ref*{#1})}}{(\ref*{#1})}}

\crefformat{equation}{#2Eq. (#1)#3}
\Crefformat{equation}{#2Eq. (#1)#3}

\Crefformat{figure}{#2Figure #1#3}
\Crefname{assumption}{Assumption}{Assumptions}
\Crefformat{assumption}{#2Assumption #1#3}
\Crefname{subsubsection}{Section}{Sections}
\crefformat{subsubsection}{#2Section #1#3}
\Crefformat{subsubsection}{#2Section #1#3}
\Crefname{alg}{Alg.}{Algs.}

\usepackage{crossreftools}
\usepackage{xparse}

\ExplSyntaxOn
\DeclareDocumentCommand{\XDeclarePairedDelimiter}{mm}
 {
  \__egreg_delimiter_clear_keys: 
  \keys_set:nn { egreg/delimiters } { #2 }
  \use:x 
   {
    \exp_not:n {\NewDocumentCommand{#1}{sO{}m} }
     {
      \exp_not:n { \IfBooleanTF{##1} }
       {
        \exp_not:N \egreg_paired_delimiter_expand:nnnn
         { \exp_not:V \l_egreg_delimiter_left_tl }
         { \exp_not:V \l_egreg_delimiter_right_tl }
         { \exp_not:n { ##3 } }
         { \exp_not:V \l_egreg_delimiter_subscript_tl }
       }
       {
        \exp_not:N \egreg_paired_delimiter_fixed:nnnnn 
         { \exp_not:n { ##2 } }
         { \exp_not:V \l_egreg_delimiter_left_tl }
         { \exp_not:V \l_egreg_delimiter_right_tl }
         { \exp_not:n { ##3 } }
         { \exp_not:V \l_egreg_delimiter_subscript_tl }
       }
     }
   }
 }

\keys_define:nn { egreg/delimiters }
 {
  left      .tl_set:N = \l_egreg_delimiter_left_tl,
  right     .tl_set:N = \l_egreg_delimiter_right_tl,
  subscript .tl_set:N = \l_egreg_delimiter_subscript_tl,
 }

\cs_new_protected:Npn \__egreg_delimiter_clear_keys:
 {
  \keys_set:nn { egreg/delimiters } { left=.,right=.,subscript={} }
 }

\cs_new_protected:Npn \egreg_paired_delimiter_expand:nnnn #1 #2 #3 #4
 {
  \mathopen{}
  \mathclose\c_group_begin_token
   \left#1
   #3
   \group_insert_after:N \c_group_end_token
   \right#2
   \tl_if_empty:nF {#4} { \c_math_subscript_token {#4} }
 }
\cs_new_protected:Npn \egreg_paired_delimiter_fixed:nnnnn #1 #2 #3 #4 #5
 {
  \mathopen{#1#2}#4\mathclose{#1#3}
  \tl_if_empty:nF {#5} { \c_math_subscript_token {#5} }
 }
\ExplSyntaxOff

\XDeclarePairedDelimiter{\supnorm}{
  left=\lVert,
  right=\rVert,
  subscript=\infty
  }

\newcommand{\ours}[1]{\textsc{MaxRL}}

\def\ddefloop#1{\ifx\ddefloop#1\else\ddef{#1}\expandafter\ddefloop\fi}
\def\ddef#1{\expandafter\def\csname bb#1\endcsname{\ensuremath{\mathbb{#1}}}}
\ddefloop ABCDEFGHIJKLMNOPQRSTUVWXYZ\ddefloop
\def\ddefloop#1{\ifx\ddefloop#1\else\ddef{#1}\expandafter\ddefloop\fi}
\def\ddef#1{\expandafter\def\csname b#1\endcsname{\ensuremath{\mathbf{#1}}}}
\ddefloop ABCDEFGHIJKLMNOPQRSTUVWXYZ\ddefloop
\def\ddef#1{\expandafter\def\csname sf#1\endcsname{\ensuremath{\mathsf{#1}}}}
\ddefloop ABCDEFGHIJKLMNOPQRSTUVWXYZ\ddefloop
\def\ddef#1{\expandafter\def\csname c#1\endcsname{\ensuremath{\mathcal{#1}}}}
\ddefloop ABCDEFGHIJKLMNOPQRSTUVWXYZ\ddefloop
\def\ddef#1{\expandafter\def\csname h#1\endcsname{\ensuremath{\widehat{#1}}}}
\ddefloop ABCDEFGHIJKLMNOPQRSTUVWXYZ\ddefloop
\def\ddef#1{\expandafter\def\csname hc#1\endcsname{\ensuremath{\widehat{\mathcal{#1}}}}}
\ddefloop ABCDEFGHIJKLMNOPQRSTUVWXYZ\ddefloop
\def\ddef#1{\expandafter\def\csname t#1\endcsname{\ensuremath{\widetilde{#1}}}}
\ddefloop ABCDEFGHIJKLMNOPQRSTUVWXYZ\ddefloop
\def\ddef#1{\expandafter\def\csname tc#1\endcsname{\ensuremath{\widetilde{\mathcal{#1}}}}}
\ddefloop ABCDEFGHIJKLMNOPQRSTUVWXYZ\ddefloop
\def\ddefloop#1{\ifx\ddefloop#1\else\ddef{#1}\expandafter\ddefloop\fi}
\def\ddef#1{\expandafter\def\csname scr#1\endcsname{\ensuremath{\mathscr{#1}}}}
\ddefloop ABCDEFGHIJKLMNOPQRSTUVWXYZ\ddefloop

\iclr{%
\hypersetup{%
  colorlinks=true,%
  citecolor=[RGB]{50,100,170},%
  linkcolor=[RGB]{50,100,170},%
  urlcolor=[RGB]{255,102,178}%
}%
}

\iclr{
\usepackage{hyperref}
\newcommand{\alghyperref}[1]{\hyperref[#1]{Alg.~\ref*{#1}}}
}

\usepackage{color-edits}
\addauthor{ys}{MidnightBlue}

\arxiv{
\usepackage[final]{showlabels}

}

\usepackage{caption}


\usepackage{accents}
\usepackage{wrapfig}
\usepackage{tikz}
\usetikzlibrary{decorations.pathreplacing}
\usepackage{varwidth}

 \addtocontents{toc}{\protect\setcounter{tocdepth}{0}}

\let\oldparagraph\paragraph
\arxiv{\renewcommand{\paragraph}[1]{\oldparagraph{#1.}}}
\iclr{\renewcommand{\paragraph}[1]{\textbf{#1.}}}

\usepackage{yfonts}

\arxiv{
    \title{Message Passing Enables Efficient Reasoning}
    \author{
      \textbf{Xuecheng Liu}$^{*}$ \quad
      \textbf{Daman Arora}$^{*}$ \quad
      \textbf{Gokul Swamy} \quad
      \textbf{Andrea Zanette} \\
      \vspace{1mm}
      Carnegie Mellon University\\
    }
}

\begin{document}

\arxiv{
\maketitle
\renewcommand{\thefootnote}{}
\footnotetext{$^*$Equal contribution.
Project webpage: \href{https://zanette-labs.github.io/mplm.github.io/}{\faGlobe}\quad
Code: \href{https://github.com/Zanette-Labs/mplm}{\faGithub}
}
\renewcommand{\thefootnote}{\arabic{footnote}}
\thispagestyle{fancy}
\fancyhead{}
\lhead{\raisebox{-0.7cm}{\includegraphics[height=0.52cm]{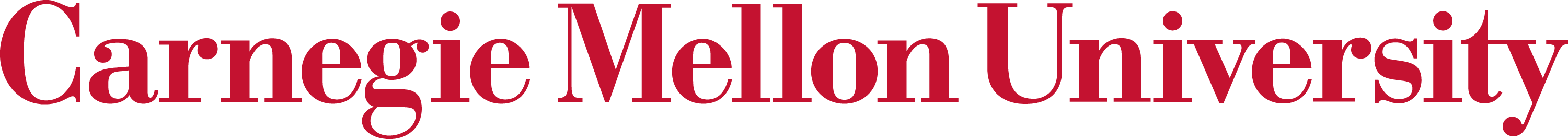}}}
\renewcommand{\headrulewidth}{0pt}
\setlength{\headheight}{18pt}
\setlength{\headsep}{3mm}
}

\newcommand{\guanning}[1]{\textcolor{cyan}{[Guanning: #1]}}


\begin{abstract}

  While inference-time scaling has improved the reasoning abilities of large language models (LLMs), the need to generate long chains-of-thought (CoTs) is a computational bottleneck. Thus, in contrast to sequential scaling methods like CoT, recent parallel scaling techniques instead use fork and join (FJ) primitives to divide work across multiple \textit{LLM threads}. However, in the fork-join paradigm, threads are typically transient and do not communicate pointwise with one another which limits scalability. To tackle this, we introduce Message Passing Language Models (MPLMs), a framework for LLM reasoning in which threads communicate directly via lightweight send and receive primitives. MPLMs enable efficient scaling through two key mechanisms: (1) reduced communication costs, achieved by avoiding redundant context sharing, and (2) preemption, which allows threads to terminate early based on partial information from their peers. 
  
  We demonstrate the promise of MPLMs on 3 classes of tasks. First, on Sudoku puzzles, we show that MPLMs require an asymptotically smaller context than both serial CoT and parallel FJ. We then fine-tune a single model to solve 25 $\times$ 25 puzzles that remain challenging for standard CoT and FJ approaches, as well as frontier reasoning models without tools. Second, on 3-SAT puzzles, the capability of preemption allows termination of unpromising branches, which results in improved efficiency. Finally, we show that appropriately prompted large pre-trained models follow the MPLM protocol, achieving competitive results on long-context question answering relative to popular fork-join approaches.

\end{abstract}

\section{Introduction}

\begin{figure*}[ht!]
    \centering
    \includegraphics[width=0.9\textwidth]{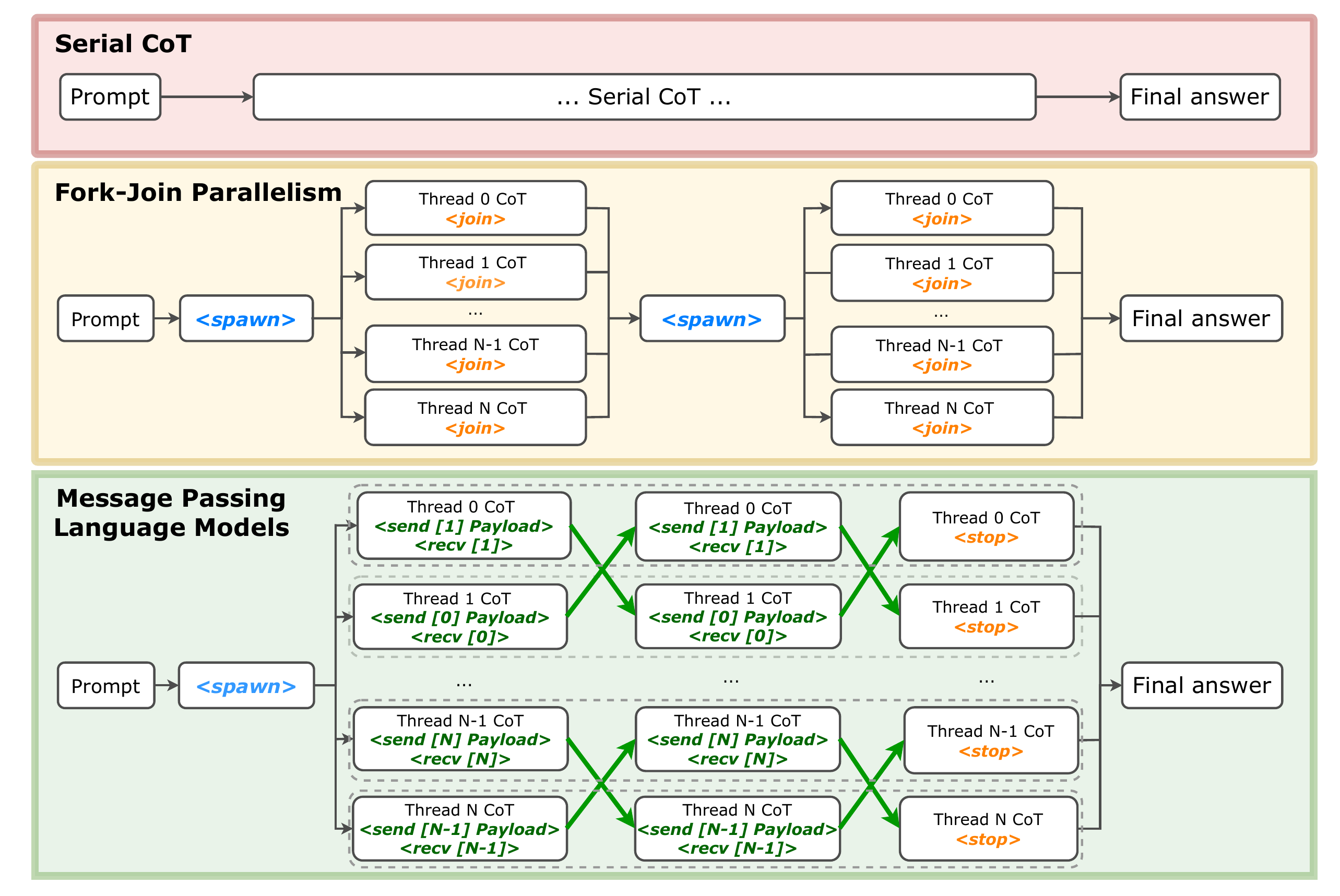}
\caption{Comparison of reasoning paradigms for scaling test-time compute.
\textbf{Top:} Serial Chain-of-Thought (CoT), where a single model generates a monolithic reasoning trace sequentially from prompt to answer.
\textbf{Middle:} Fork-Join parallelism, where multiple independent transient CoT threads are spawned in parallel and periodically synchronized via global join operations before continuing.
\textbf{Bottom:} Message Passing Language Models (MPLMs), where concurrently decoded threads communicate directly through explicit \texttt{send} and \texttt{recv} primitives and maintain a persistent context, enabling fine-grained coordination and information sharing throughout inference. }
    \label{fig:main_figure}
    \vspace{-0.5em}
\end{figure*}

\emph{Large reasoning models} \citep{Guo_2025, openai2024openaio1card} enhance the reasoning capabilities of foundation models by allocating extra compute at test time. This compute is devoted to a \emph{chain of thought} (CoT)---an intermediate sequence of reasoning steps. The dominant reinforcement-learning-based training paradigm used to train reasoning models \citep{Guo_2025} yields \emph{monolithic} CoTs, where the model explores alternative strategies, backtracks on inconsistencies, verifies intermediate results, and aggregates partial solutions into a final answer. However, this process is inherently \emph{serial}, as tokens must be generated autoregressively, creating a computational bottleneck. Moreover, transformer-based architectures \citep{Vaswani2017AttentionIA} impose practical limits: positional embeddings \citep{su2023roformerenhancedtransformerrotary} constrain usable context and the quadratic cost of self-attention reduces returns on additional compute. As we imagine an exciting future where reasoning models can tackle complex problems, such as engineering a new product, advancing scientific research, or even running a company, autoregressively generated CoTs may ultimately prove to be a limited approach.

Not all reasoning needs to unfold as a single sequential chain. For example, when designing a complex software system, work is often divided across teams: one team may focus on the database layer, another on the user interface, and another on the networking layer and they all work in parallel, coordinating when necessary. An initial attempt to this paradigm of scaling inference-time compute has been investigated in recent literature \citep{pan2025learning, zheng2025parallelr1parallelthinkingreinforcement} where an LLM can \textit{fork} its own execution and decode multiple threads in parallel, before the resulting computation of the threads is aggregated by the parent thread.

While this fork-join paradigm reduces latency and alleviates context pressure, it retains critical limitations: \emph{all coordination is centralized} and \emph{context for workers is transient}. Parallel workers remain isolated during execution, their contexts are temporary, and all information must flow through the parent thread, which introduces both communication bottlenecks and implicit serialization. Conceptually, this is akin to an organization in which every decision-no matter how local or routine-must be routed through a single CEO: such a structure is inherently brittle, limits parallelism, and defeats the purpose of delegation.

In response, we propose \textbf{Message Passing Language Models} (MPLMs), a framework that decomposes reasoning into persistent, semi-independent threads that coordinate via explicit point-to-point message passing. By distributing reasoning across multiple interacting workers that communicate through explicit message passing, MPLMs enable more efficient use of test-time compute and improved scalability as problem complexity grows. Unlike prior multi-agent or debate-based approaches, which rely on independently acting agents with pre-defined roles and fixed communication protocols, MPLMs formulate reasoning as a single, unified process that is decomposed into communicating sub-threads. Rather than optimizing separate agents for different objectives, a single model dynamically decides when, with whom, and what to communicate using its own learned reasoning capabilities, allowing intricate communication patterns and task-specific structures to emerge organically.



 We instantiate these ideas on two structured reasoning testbeds: Sudoku and 3-SAT puzzles. These problems are deterministic and naturally parallelizable which allows us to scale complexity incrementally in a way that cleanly isolates the benefits of our reasoning decomposition from confounding factors. Efficient Sudoku solutions exploits the benefits of sparse, local communication, while 3-SAT utilizes a complementary advantage of preemption in search.  We show more favorable scaling with complexity of the problem, more specifically the size of the grid in Sudoku and the number of variables in 3-SAT as compared to Serial or Fork-Join baselines. Particularly impressively, we are able to scale to 25 $\times$ 25 Sudoku grids that remain \textbf{out of reach even for frontier reasoning models} such as GPT-5 Pro. 

Finally, we demonstrate preliminary evidence of the ability of large pretrained models to use the MPLM framework when prompted, without the need for any training. On LongBench-v2 \citep{longbenchv2}, we show promising results highlighting improved efficiency as compared to RLMs \citep{zhang2026recursivelanguagemodels}, a contemporary fork-join approach.

In summary, our contributions are three-fold: 
\begin{enumerate}
\item We introduce MPLMs, a framework for message passing between parallel LLM threads and analyzing the key \textit{mechanisms} through which they unlock improved efficiency for language model inference.
\item We demonstrate \textit{empirical scalability} on Sudoku puzzles well beyond currently existing serial CoT and Fork-join parallelism and improved efficiency of SAT puzzles.
\item We provide evidence on the ability of large pretrained models to \textit{route information appropriately} using the MPLM framework on Long Context Question Answering.
\end{enumerate}

\section{Preliminaries}

\textbf{Language Model}\;
A language model (LM) computes a distribution over tokens $y$ in a vocabulary $\mathcal{V}$, given a prefix $x = (x_1, x_2, \cdots, x_N) \in \mathcal{V}^*$. Formally, an LM defines the conditional distribution $p(y | x)$ over the vocabulary $\mathcal{V}$. In the standard setting, generation in LMs is \textit{autoregressive}, that is, given a prefix $x$, the model generates tokens $y_1, y_2, \cdots, y_M$ based on a sampler over the conditional distribution over tokens.  Modern LLMs are based on the Transformer \citep{Vaswani2017AttentionIA} architecture, where each token attends to all previous tokens in the context window causing both training and inference costs to scale quadratically with the sequence length, making long-context generation expensive.

\noindent\textbf{Batched inference} \;
Batched inference refers to serving multiple generation requests in a single forward pass (or in tightly packed micro-batches) to improve device utilization and throughput. Frameworks such as vLLM \citep{kwon2023efficient} and SGLang \citep{zheng2024sglang} are highly optimized for batch inference and exemplify this design pattern by combining runtime scheduling with KV-cache-aware execution allowing high-throughput inference.

\noindent\textbf{MPI} \; The Message Passing Interface (MPI)  is a standardized framework for parallel and distributed computation based on explicit message passing between processes, exposing primitives such as send/receive, synchronization, and dynamic process management. Classic references such as \citet{10.5555/207387, 10.5555/552013} describe how persistent workers exchanging messages enables decentralized computation. Our work adapts these principles as an abstraction for distributed reasoning among LLM threads.

\section{Message Passing Language Models}
\label{sec:MPLM}

In traditional LLM inference, a model generates a \emph{single} stream of tokens autoregressively in response to a user's request. In contrast, MPLMs create multiple \emph{LLM threads} for handling sub-tasks, which are indexed by an ID and associated with distinct prompts and contexts.
These threads are decoded concurrently through batched inference
\citep{kwon2023efficient, zheng2024sglang}.

\subsection{Execution Control Directives}
The foundation of MPLMs is a set of \textit{directives}: explicit instructions that control execution. The model itself generates special tagged sub-strings that the framework interprets as directives. These play a role analogous to process creation, point-to-point communication, synchronization, and termination in message-passing systems. In MPLMs, we have 4 fundamental directives: \texttt{spawn}, \texttt{send}, \texttt{receive} and \texttt{stop}:

\begin{itemize}[leftmargin=*, itemsep=0pt, topsep=0pt]
\item \textbf{Spawn} The spawn directive requires the model to generate a string of the format:

{\centering\ttfamily <spawn[$id_1, id_2, \dots, id_N$]> prompt
</spawn>\par}

This generates $N$ new LLM threads (similar to forking) with the corresponding textual IDs and the `prompt' concatenated in their individual contexts.

\item \textbf{Send} The send directive is one of the two primitives responsible for handling message-passing. To generate a send, the model has to generate a string of the form:

{\centering\ttfamily <send[$id_1, id_2, \dots, id_N$]> message </send>\par}

This sends the {\ttfamily{message}} mentioned in the directive to LLM threads with IDs $id_1, id_2 ... id_N$.

\item \textbf{Receive} The receive directive is the second primitive responsible for handling message-passing. To generate a receive, the model has to generate a string like:

{\centering\ttfamily <recv[$id_1, id_2, \dots, id_N$]>\par}

The behavior of this primitive is: wait for messages from threads $id_1, id_2 ... id_N$.

\item \textbf{Stop} This directive is responsible for ending the execution of an LLM thread. A stop is triggered when a thread generates a string of the form \texttt{<stop>}. Whenever a stop is generated by a thread its execution is permanently stopped by the inference engine.
\end{itemize}

\subsection{Inference Logic}

Now, we describe how these directives are instantiated on a standard batched inference engine. Standard batched inference engines (e.g., vLLM) treat each decoding request independently. MPLMs instead require \emph{coordinated} parallel decoding: threads may spawn new threads, send messages, and block on receives. We implement this coordination via a lightweight \emph{controller} layered on top of the inference engine $\mathcal{E}$ while leaving the underlying token sampling and KV-cache management unchanged. This controller can be visualized as a mini-operating system working on the substrate which is the inference engine $\mathcal{E}$.

The controller repeatedly invokes $\mathcal{E}$ to decode all \textit{runnable} threads \emph{in parallel} until they reach either (a) a directive boundary, (b) EOS, or (c) a maximum step budget. A directive boundary occurs when the emitted token stream matches a directive format (\texttt{spawn}, \texttt{send}, \texttt{recv}, or \texttt{stop}). This is implemented by specifying special stop-tokens to $\mathcal{E}$. When a thread emits a \texttt{send} directive, the specified message is delivered to the target threads and decoding continues. On the other hand, upon emitting \texttt{recv}, the thread is blocked until it has received messages from all the relevant threads listed in the directive. For a more detailed pseudo-code reference, refer Section \ref{sec:pseudocode} of the Appendix. 

This framework also gives freedom in terms of certain design choices such as the implementation of the \texttt{recv} directive. In particular, we explore a wait-for-all and a wait-for-any implementation. Additionally, it also naturally allows a thread to \textit{re-spawn} itself with a compressed task summary which leads to higher efficiency. Finally, we discuss how the ability of \textit{preempting} another thread is natively supported by the MPLM framework. We defer a detailed discussion on these aspects to Appendix \ref{app:mplm_additional} for clarity of exposition.

\section{MPLMs Enable Efficient Reasoning }

Next, we formalize the intuitive notion that reduced aggregation costs leads to efficiency in MPLMs. In particular, we design a simplified setup below to theoretically show that MPLMs can be significantly more efficient than other paradigms such as Serial CoT and Fork-Join (FJ).

\noindent\textbf{Setup.}\; We consider a task that proceeds in $T$ iterations. In each iteration, $N$ workers (ranks) run in parallel, each solving a local sub-task and producing at most $W$ tokens. To advance to the next iteration, each worker requires information from at most $k$ workers from the previous iteration. This information is exchanged via messages of size at most $M$ tokens. Thus, rather than requiring access to all previous computations, each iteration depends only on a small, localized subset of the messages constructed in the previous iteration. We compare three paradigms to operationalize this particular task keeping in mind their natural constraints:

\begin{enumerate}[leftmargin=*, itemsep=0pt, topsep=0pt]
    \item \textbf{Serial CoT:} This is the standard way in modern reasoning models reason, that is, using a Serialized Chain-of-Thought \citep{Guo_2025, openai2024openaio1card}. Since there is no parallelization, all the work of individual ranks is done sequentially for each iteration.
    \item \textbf{Fork-Join (FJ):} In this paradigm \citep{pan2025learning, zheng2025parallelr1parallelthinkingreinforcement}, a model can spawn threads for various sub-tasks and then executes them in parallel. At the end of an iteration, a master rank collects all messages, constructs new per-rank prompts, and spawns fresh workers for the next iteration.
    \item \textbf{Message passing (MPLM):} Each rank maintains a persistent context throughout iterations and communicates point-to-point with its $k$ neighbors.
\end{enumerate}

\begin{assumption}[Bounded local processing]
\label{assumption:comm_dominated}
The per-iteration local reasoning grows linearly with the total message
volume received by the thread, i.e., $W = \mathcal{O}(kM)$. This sparsity condition is naturally satisfied in many structured reasoning tasks. For example, in code generation, modules typically depend on a small set of interfaces rather than the full codebase.
\end{assumption}

Now, we analyze the maximum context required for each of these methods since context length is the fundamental bottleneck in scaling inference compute. Denote the maximum context length required by the Serial, FJ and MPLM paradigms as $C_{\max}^{\mathrm{Serial}}, C_{\max}^{\mathrm{FJ}},$ and  $C_{\max}^{\mathrm{MPLM}}$.

\begin{theorem}[\textit{\textbf{Maximum context for Serial v. FJ v. MPLM}}]
\label{thm:max_context}
\textup{[Proof: Appendix~\ref{sec:extended_theory}]}
\[
C_{\max}^{\mathrm{Serial}} = \bigO{(TNkM)},\;\;\;\;
C_{\max}^{\mathrm{FJ}} = \bigO{(TNM)} \;\; \text{and} \;\;
C_{\max}^{\mathrm{MPLM}} = \bigO{(TkM)}.
\]
\end{theorem}

We observe that Serial CoT incurs the highest context requirements due to serialization. The Fork-join execution improves this by parallelizing the reasoning per-rank, but still incurs a centralized aggregation cost that scales with $N$. In contrast, MPLMs exploit sparse communication yielding a maximum context requirement that is smaller by a factor of $\Theta(N/k)$ compared to FJ. This factor is especially significant when communication is sparse. This provides a theoretical justification for the \textbf{context efficiency} of MPLMs. 

\section{Experiments}
\begin{figure}[h!]
    \centering
    \begin{subfigure}{0.45\textwidth}
        \centering
        \includegraphics[width=\linewidth]{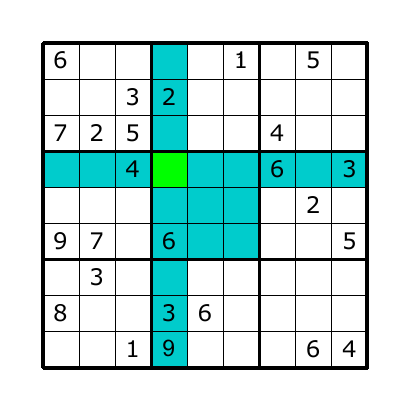}
    \end{subfigure}
    \hfill
    \begin{subfigure}{0.45\textwidth}
        \centering
        \includegraphics[width=\linewidth]{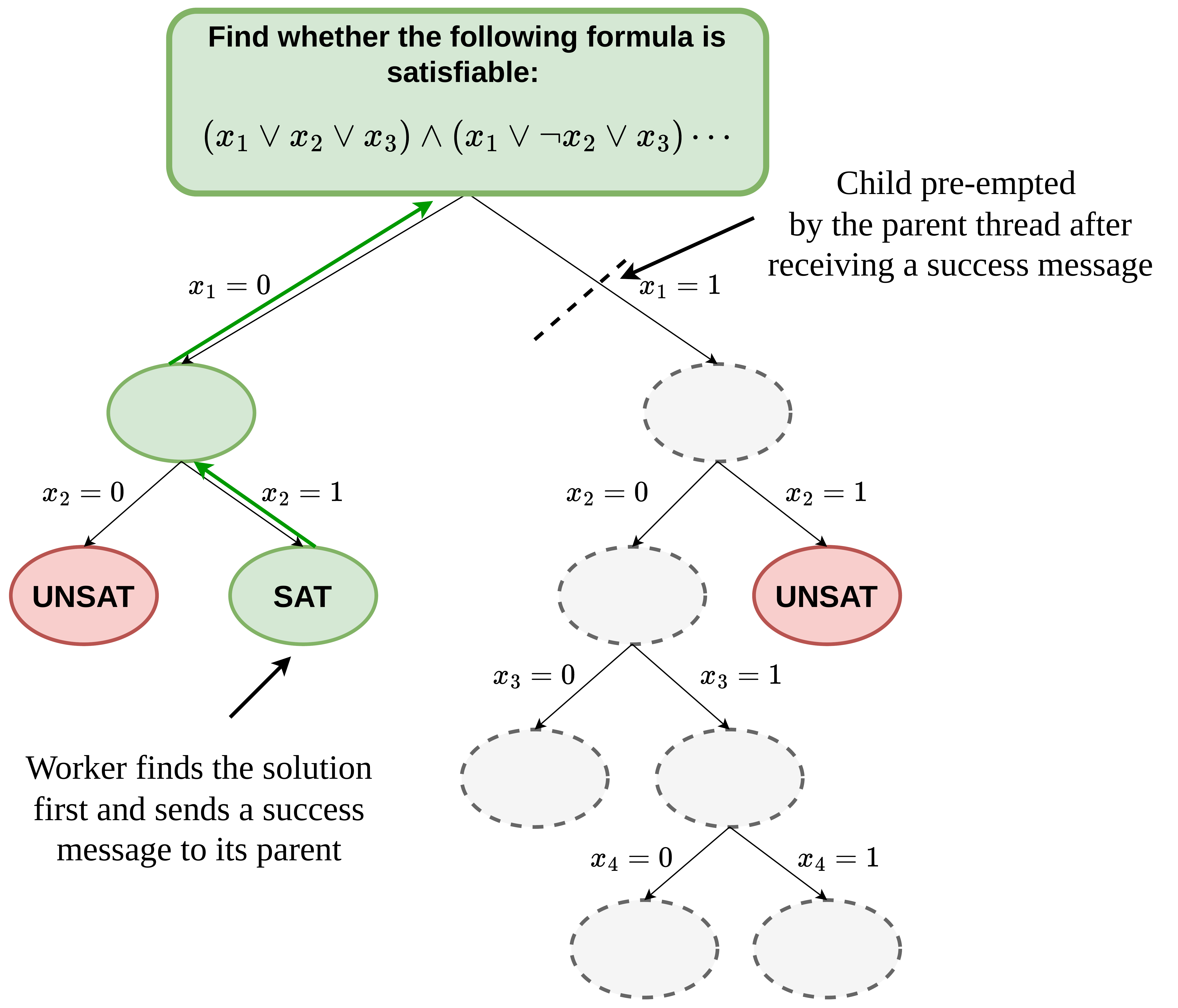}
    \end{subfigure}
        \caption{\textbf{Left:} This figure describes the communication pattern for a $9\times 9$ Sudoku grid. Consider the worker for the cell (4, 4) highlighted in green. This cell only sends and receives messages from its neighbors highlighted in blue, reducing context requirement. \textbf{Right:} This figure describes how preemption can improve efficiency in 3-SAT problems where a search branch that finds a solution can send a message its parent without having to wait for other threads to finish their computation.}
    \label{fig:sudoku_example}
\end{figure}
We ground MPLMs in 2 structured reasoning tasks: Sudoku and 3-SAT Puzzles. These experiments allow us to empirically validate the efficiency gains on sparse problems we proved above. For both experiments, we teach a model, using supervised fine-tuning, to reason in a way that allows it to exploit the primitives afforded by MPLMs, more specifically the \textit{send} and \textit{recv} directives. We describe more details below:

\textbf{Training Procedure:}\; We simulate the message-passing reasoning process in Python and use the resulting trajectories to generate Chain-of-Thought training data. For Sudoku, we simulate execution of the naked-singles strategy iteratively; for SAT, we follow DPLL algorithm \citep{davis1962machine}. Example CoTs for each method can be found in Appendix \cref{app:sample_cot,app:sample_cot_fj,app:sample_cot_serial,app:sample_cot_sat_mplm,app:sample_cot_sat_fj,app:sample_cot_sat_serial}. Using these datasets, we perform supervised fine-tuning (SFT) on Qwen3-0.6B-Base \citep{yang2025qwen3technicalreport}. Each training run cost less than 48 H100 hours. More details can be found in Appendix \ref{app:data_gen}.

\textbf{Metrics:}\; To measure the performance of our method, we consider metrics including (i) \textbf{Accuracy:} the percentage of puzzles solved successfully; (ii) \textbf{Latency:} the average time taken to solve a puzzle in seconds; (iii) \textbf{Maximum context:} the maximum context required to train one puzzle; and (iv) \textbf{Sequential tokens:} the maximum number of causally dependent tokens that must be processed sequentially by the model, as used by \cite{pan2025learning}.

\textbf{Baselines:}\; For baselines we also generate data for (i) Serialized \textbf{(Serial)} Chain-of-Thought which generates every token autoregressively with no parallelism and (ii) Fork-Join \textbf{(FJ)} parallelism as implemented by \cite{pan2025learning} where a parent thread iteratively spawns child threads and the parent thread combines the final responses of the children. Additionally, we evaluate closed-source reasoning models on Sudoku puzzles, including DeepSeek-R1 and GPT-5 Pro, using their respective APIs to assess how performance scales with problem complexity. We disable tool usage to isolate the model's reasoning behavior without external computation. 

\subsection{Sudoku}

\noindent\textbf{Setup}\; Generalized Sudoku is an $N^2\times N^2$ constraint satisfaction problem requiring unique values across rows, columns, and sub-grids; we restrict to puzzles solvable by iteratively filling cells with only one remaining feasible value (naked-singles). We generate 1000 puzzles of sizes $4\times 4$, $9\times 9$, $16 \times 16$ and $25 \times 25$ for training and 200 puzzles for evaluation. See Appendix \ref{app:data_gen} for more details.


\subsubsection{A Message-Passing Approach to Sudoku}

\begin{figure*}[t!]
    \centering
    \includegraphics[width=0.9\textwidth]{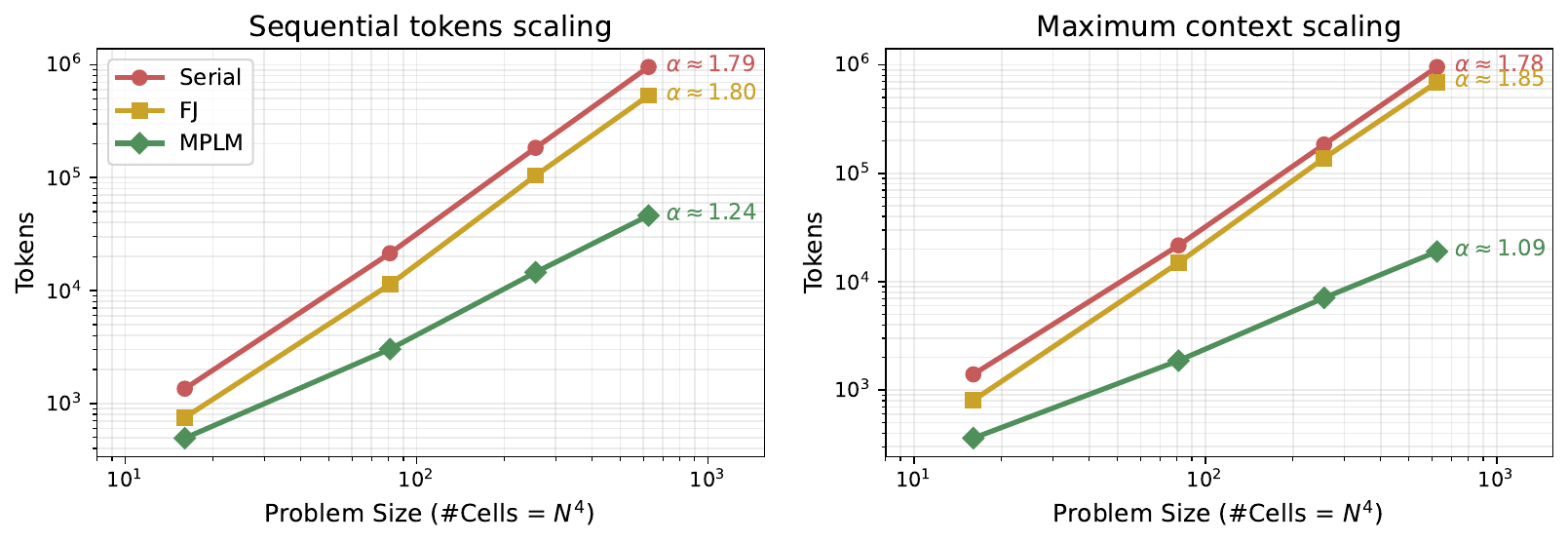}
    \caption{\textbf{Scaling} of sequential tokens (left) and maximum context tokens (right) required to solve a Sudoku puzzle as a function of the problem size ($N^4$). Curves are shown on a log-log axes with fitted power-law exponents $\alpha$. Message Passing (MPLM) exhibits significantly \textbf{lower scaling exponents} than the Serial and Fork-Join (FJ) reasoning paradigms for both sequential tokens and maximum context required, indicating both better efficiency and the ability to solve larger problems.}
    \label{fig:main_results}
\end{figure*}

In our MPLM design for Sudoku, we have a single agent for each cell in the Sudoku puzzle, which leads to $N^4$ workers for a $N^2\times N^2$ puzzle. First, the parent thread (ID `0') spawns $N^4$ child threads (IDs $1, 2,\cdots, N^4$) using the \texttt{<spawn>} directive and sends them the initial state of the puzzle. Then, each worker computes the possible set of values for its corresponding cell. For instance, in a $9\times 9$ puzzle, worker `12' would compute the possible values possible for cell $(1, 3)$. After computing this set, there are 2 possibilities:

\begin{enumerate}[leftmargin=*, topsep=0pt, itemsep=0pt]
\item \textbf{Case 1:} There are multiple possibilities for the current cell. In this case, it is necessary for the worker to gather more information before proceeding. Therefore, it generates the \texttt{<recv>} token, signaling that it is waiting for messages from its its `neighbors'. Neighbors here refer to the workers in control of cells in the same row, column, or sub-grid.
\item \textbf{Case 2:} There is only 1 possibility for the current cell. In this case, the worker sends a message using the \texttt{<send [list of IDs]>} directive to the parent thread (ID `0') and all its `neighbors' containing the value filled in the cell. For an $N^2\times N^2$ grid, each cell has $3N^2-2N-1$ neighbors (see Figure \ref{fig:sudoku_example} for an example). After sending the message, the model generates \texttt{<stop>} to end its execution.
\end{enumerate}

After all threads have stopped, the parent thread collates all the messages and generates the final solved Sudoku, which is parsed to evaluate the success/failure of the model.


\begin{table*}[t!]
\centering
\begin{tabular}{ccccc}
\hline
\textbf{Method} & \textbf{$n=2$ ($4\times4$)} & \textbf{$n=3$ ($9\times9$)} & \textbf{$n=4$ ($16\times16$)} & \textbf{$n=5$ ($25\times25$)} \\
\hline
MPLM         & \textbf{2.105s (100\%)} & \textbf{14.939s (100\%)} & \textbf{117.684s (92\%)} & \textbf{1017.263s (72\%)}   \\
FJ          & 2.940s (100\%) & 59.562s (93.00\%)              & \textcolor{red}{\xmark} & \textcolor{red}{\xmark} \\
Serial  & 4.545s (99\%)  & \textcolor{red}{\xmark}              & \textcolor{red}{\xmark} & \textcolor{red}{\xmark} \\

DeepSeek R1 & 144.271s (90\%)& 765.982s (0\%)& 597.939s (0\%) & 431.751s (0\%) \\
GPT-5 Pro     & (100\%) & (100\%)              & (45\%) & (20\%) \\
\hline
\end{tabular}
\caption{Avg. latency and accuracy across Sudoku sizes. We report wall-clock inference time and solution accuracy. A red cross (\textcolor{red}{\xmark}) denotes infeasible training due to context or compute constraints. MPLM remains feasible and accurate at larger grid sizes, while FJ, Serial baselines, and frontier models fail to scale. Due to the high cost of these models, we limit our evaluation to 20 queries per setting, and use the batched API which doesn't provide latency metrics. }
\label{tab:latency_and_accuracy}
\vspace{-0.5em}
\end{table*}

\subsubsection{Sudoku Results}
\textbf{MPLMs fundamentally improves scaling.}
We note that MPLMs operate in a qualitatively different scaling regime compared to both \textbf{Serial} and \textbf{Fork-Join (FJ)} reasoning as seen in Figure ~\ref{fig:main_results}. Serial and FJ approaches exhibit worse power-law behavior, with both sequential token counts and maximum context requirements scaling with higher exponents of $\alpha \approx 1.8$. In contrast, MPLM achieves \textit{substantially improved scaling} exponents ($\alpha \approx 1.2$ for sequential tokens and $\alpha \approx 1.1$ for maximum context) owing to its \textit{decentralized context}. This reflects a structural change in computation: MPLMs distribute reasoning across persistent workers that maintain local state and communicate point-to-point. By avoiding repeated global aggregation, MPLMs improve the asymptotic growth of both sequential computation and context usage. We also note that different paradigms might admit different CoT implementations, and alternative designs may improve constant factors; however, such changes \textit{do not affect the asymptotic scaling trends} that form the basis of our comparison. A more detailed theoretical analysis is provided in Appendix~\ref{app:sudoku_theory}.


\noindent\textbf{Message passing enable capabilities beyond frontier reasoning models.}
The improved scaling behavior of MPLM reasoning results in qualitative gains in problem-solving capability. In standard $N=3$ Sudoku puzzles, we find that DeepSeek-R1 fails to reliably solve instances whereas \textit{GPT-5 Pro} struggles in $25\times 25$ Sudoku puzzles solving just 4 of 20 puzzles. In contrast, MPLM successfully solves $72\%$ of $25\times 25$ instances after being trained on just 1000 puzzles. We note that this comparison is not intended as a direct head-to-head evaluation however, the results illustrate what the MPLM framework can unlock when a model is trained to exploit structured parallelism, a capability that is difficult to achieve within serial CoT reasoning regardless of model scale. Together, these results show that MPLMs enable a new level of structured problem-solving.  Additionally, we show that respawning can also improve efficiency for Sudoku in Appendix \ref{app:respawn}. We also conduct additional measurements of quantities such as KV-Cache hit rate and GPU utilization in Appendix \ref{app:profiling}.

\subsection{3-SAT}

\textbf{Setup.}\; We also study Boolean satisfiability (3-SAT), where \textit{preemption} enabled by \textit{asynchronous message passing} is critical. Solving SAT problems requires a search over possible assignments to variables. Our key insight to improve efficiency is that, once a satisfying assignment is found in one branch of the search tree, continuing computation in other branches is unnecessary. Such behavior is not possible with the Fork-Join paradigm, where the master only aggregates worker's responses after they have finished. In contrast, preemption can be easily implemented in the MPLM scaffold, where a worker can immediately send a message to its parent (which is actively waiting to receive messages) and the parent can decide to preempt the execution of other workers. To empirically verify this preemption capability, we generate random 3-SAT instances with variable counts ranging from 8 to 20 and train the model to solve the problems using SFT on the CoT. Training details are provided in Appendix \ref{app:sat_accuracy}.

\subsubsection{A Message-Passing Approach to 3-SAT}
For 3-SAT, our MPLM design follows a distributed DPLL procedure \citep{davis1962machine}. Each worker first performs unit propagation on its local subproblem. If the formula is resolved immediately, the model generates either \texttt{SAT} or \texttt{UNSAT}. Otherwise, it branches on a variable $x_i$ and spawns 2 children corresponding to the assignments $x_i=0$ and $x_i=1$, and waits for messages from its descendants. MPLM supports preemption through asynchronous message passing. When a thread stops, the controller terminates its entire sub-tree to improve efficiency and avoid unnecessary computation. In particular, a parent does not need to wait for all children to complete before making progress. Instead, it reacts to \textbf{the child that returns first}. If a child reports \texttt{SAT}, the parent immediately terminates the search. This behavior propagates upward recursively to the root node as shown in Figure \ref{fig:sudoku_example}.

\begin{figure*}[ht!]
    \centering
    \includegraphics[width=0.9\textwidth]{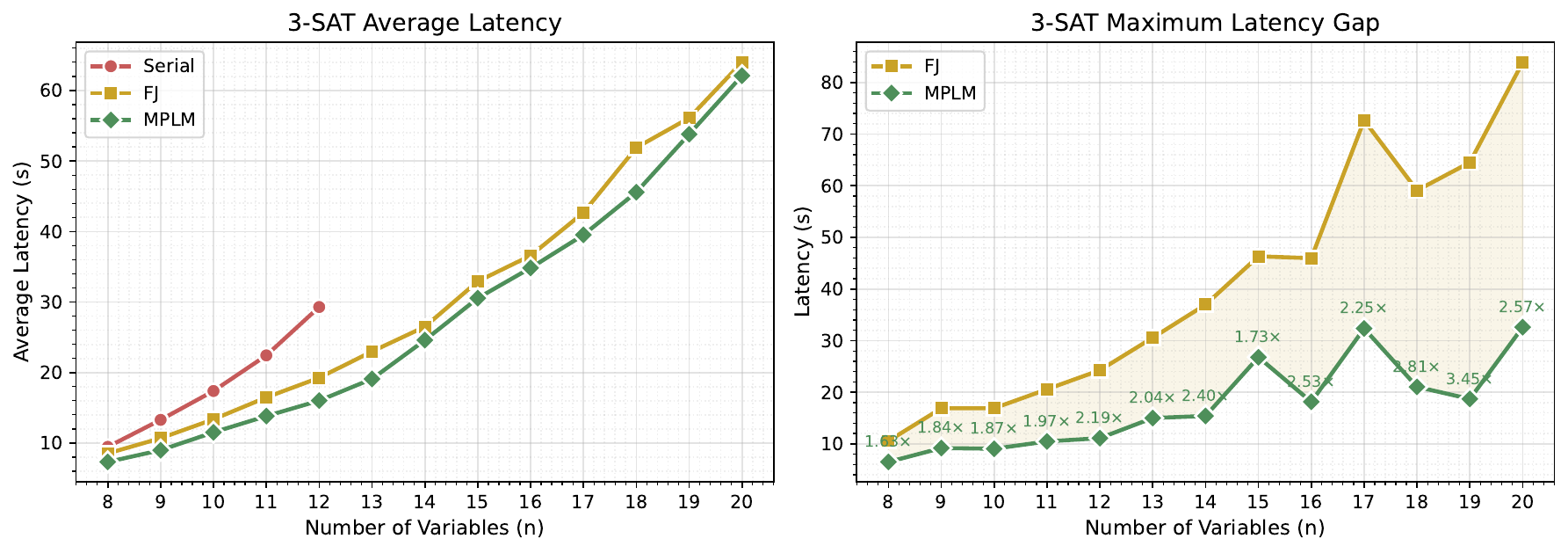}
    \caption{\textbf{Left:} Average inference latency across 100 test problems per variable count for MPLM, FJ, and Serial. \textbf{Right:} Latency and speedup on the maximum-speedup problem at each variable count, observed in the problems whose search tree is highly unbalanced. Note that Serial can’t be measured after 12 variables because it exhausts the context window.}
    \label{fig:sat_results}
\end{figure*}

\subsubsection{3-SAT Results}

\textbf{MPLMs reduce unnecessary computation through preemption.}\;
Figure~\ref{fig:sat_results} shows the latency results on 3-SAT for increasing complexity of the problem. The left plot shows the \emph{average latency} over the full test set for each number of variables, and the right plot shows the \emph{maximum speedup case} at each variable count, i.e., the test instance for which MPLM achieves the largest latency advantage over Fork-Join. Accuracy results are in Appendix \ref{app:sat_accuracy}. We observe that MPLMs are faster than other baselines for all sizes of the problem. Additionally, as shown in the right plot of Figure~\ref{fig:sat_results}, the greatest gains arise when the SAT search tree is highly unbalanced, where the two child branches of a parent differ dramatically in the amount of computation required before reaching a decisive node. In such cases, Fork-Join suffers from its centralized synchronization structure, and its overall latency is determined by the time required to reach the last leaf node. In contrast, MPLM allows for early termination, so the overall latency is determined by the time to the \emph{first} satisfying node rather than the \emph{last} explored node.

\subsection{Long Context QA}

\begin{table*}[t!]
\centering
\scriptsize
\setlength{\tabcolsep}{3pt}
\resizebox{\textwidth}{!}{
\begin{tabular}{lcccccccc}
\toprule
\multirow{2}{*}{\textbf{Partition}} 
& \multicolumn{4}{c}{\textbf{Qwen3-30B-A3B}} 
& \multicolumn{4}{c}{\textbf{Qwen3.6-35B-A3B}} \\
\cmidrule(lr){2-5} \cmidrule(lr){6-9}
& \makecell{MPLM \\ Accuracy ($\uparrow$)} 
& \makecell{MPLM \\ Latency ($\downarrow$)} 
& \makecell{RLM \\ Accuracy ($\uparrow$)} 
& \makecell{RLM \\ Latency ($\downarrow$)}
& \makecell{MPLM \\ Accuracy ($\uparrow$)}
& \makecell{MPLM \\ Latency ($\downarrow$)}
& \makecell{RLM \\ Accuracy ($\uparrow$)} 
& \makecell{RLM \\ Latency ($\downarrow$)} \\

\midrule
\textit{I. Single-Document QA} 
& \textbf{38.1\%} & \textbf{58.2s} & 36.0\% & 117.3s
& \textbf{48.0\%} & \textbf{107.9s} & 47.4\% & 233.0s \\

\textit{II. Multi-Document QA} 
& \textbf{35.0\%} & \textbf{58.8s} & 21.6\% & 104.8s
& \textbf{47.2\%} & \textbf{104.2s} & 40.0\% & 235.4s \\

\textit{III. Long In-context Learning} 
& \textbf{34.9\%} & \textbf{68.3s} & 26.5\% & 98.5s
& \textbf{49.4\%} & \textbf{118.1s} & 38.3\% & 327.6s \\

\textit{IV. Long-dialogue History Understanding} 
& \textbf{51.3\%} & \textbf{60.1s} & 32.0\% & 92.6s
& \textbf{61.6\%} & 94.1s & 59.0\% & \textbf{72.7s} \\

\textit{V. Code Repository Understanding} 
& \textbf{35.0\%} & \textbf{65.5s} & 32.5\% & 93.3s
& 28.0\% & \textbf{56.7s} & \textbf{52.0\%} & 187.8s \\

\textit{VI. Long Structured Data Understanding} 
& \textbf{42.4\%} & \textbf{64.2s} & 28.1\% & 99.8s
& 39.4\% & \textbf{103.5s} & \textbf{66.7\%} & 104.2s \\

\midrule
\textbf{Average} 
& \textbf{37.8\%} & \textbf{61.3s} & 29.7\% & 105.7s
& 46.5\% & \textbf{102.2s} & \textbf{46.7\%} & 223.5s \\

\bottomrule
\end{tabular}
}
\caption{LongBench-v2 results aggregated by major task category using task-count weighted averages. We compare MPLM against the RLM fork-join baseline using Qwen3-30B-A3B and Qwen3.6-35B-A3B without fine-tuning.}
\label{tab:longbenchv2}
\vspace{-0.5em}
\end{table*}

\textbf{Setup.}\; While Sudoku and SAT isolate the algorithmic benefits of message passing in controlled settings, we also evaluate MPLMs on a more naturalistic long-context reasoning benchmark. We use LongBench-v2~\citep{longbenchv2}, which contains 503 multiple-choice questions with contexts ranging from 8K to 2M words. The benchmark covers a diverse set of long-context tasks, including single-document QA, multi-document QA, long in-context learning, long-dialogue history understanding, code repository understanding, and long structured-data understanding. This setting is a natural testbed for MPLMs because solving long-context QA can be tackled by decomposing the context to multiple threads and cross-referencing evidence distributed across different chunks of the context. We describe our approach in detail below.

\subsubsection{A Message-Passing Approach to Long Context Tasks}
 The parent thread first spawns worker for different parts of the large context based on a maximum threshold. For instance, if the threshold is 25K tokens and the entire context is 200K tokens, then 8 parallel threads are spawned.  Then, each worker retrieves its assigned chunk into its local context and sends a concise summary to the parent. The parent then reasons over the aggregated summaries and can \textbf{selectively} and \textbf{iteratively} communicate with workers. If the parent needs more precise evidence from a specific region of the context, it sends a targeted query to the relevant workers; if the summaries are insufficient, it can query a broader subset or all workers. Once enough evidence has been collected, the parent terminates and produces the final answer. An example can be seen in Appendix \ref{app:long_qa}.

\noindent\textbf{Baseline.}\; We compare against Recursive Language Models (RLMs)~\citep{rlm}, a reasoning framework which uses fork-join style parallelism to spawn multiple threads which are equipped with specialized tools to gather context and write code. RLM workers are stateless, after returning their response to the parent, they do not maintain persistent local state and cannot be queried again through point-to-point communication. This makes RLM a strong and relevant baseline for evaluating whether persistent workers and selective message passing provide benefits beyond generic parallel decomposition.

\subsubsection{LongBench-v2 Results}

\textbf{Prompted models can already use MPLM directives.}\;
Unlike our Sudoku and SAT experiments, this evaluation does not rely on any task-specific training. Both Qwen3-30B-A3B and Qwen3.6-35B-A3B are prompted to use the MPLM scaffold directly. The results therefore indicate that sufficiently capable models can already follow message-passing directives and make use of persistent point-to-point communication at inference time. For future work, these prompted MPLM traces can be used as initial data for training stronger MPLM policies through supervised fine-tuning or reinforcement learning.

\noindent\textbf{Persistent workers and point-to-point communication enable efficient iterative evidence aggregation.}\;
As shown in Table~\ref{tab:longbenchv2}, MPLM improves average accuracy from $29.7\%$ to $37.8\%$ on Qwen3-30B-A3B while reducing average latency from $105.7$s to $61.3$s, which corresponds to approximately $1.7\times$ lower latency. On the stronger Qwen3.6-35B-A3B model, MPLM achieves comparable average accuracy to RLM ($46.5\%$ vs. $46.7\%$) while reducing average latency from $223.5$s to $102.2$s, which corresponds to approximately $2.2\times$ lower latency. We observe that in MPLM, the parent can repeatedly query the same worker for finer-grained details without reloading that part of the document. At the same time, point-to-point communication allows the parent to contact only the workers that are relevant to the current uncertainty. This reduces unnecessary context consumption and makes iterative evidence refinement more efficient. In contrast, fork-join workers are transient. Once a worker returns its response, its local context is discarded, making subsequent follow-up queries require re-spawning workers and reconstructing context.

\section{Related Literature}
\textbf{Scaling test-time compute.}\;
Various techniques such as Self-Consistency \citep{wang2023selfconsistency}, Best-of-N \citep{cobbe2021trainingverifierssolvemath}, Monte Carlo Tree Search \citep{xie2024monte}, Tree-of-Thoughts \citep{yao2023tree} have been proposed to improve performance by scaling test-time compute. Another direction of scaling is sequential scaling which includes techniques such as Chain-of-Thoughts \citep{wei2022chain}, Self-Refine \citep{madaan2023selfrefine}, Self-Correction \citep{kumar2024traininglanguagemodelsselfcorrect}. However, with the advent of large reasoning models, long Chain-of-Thoughts learnt using reinforcement learning \citep{shao2024deepseekmathpushinglimitsmathematical, zelikman2022star} have been more effective in scaling test-time compute \citep{Guo_2025} compared to other contemporary methods.

\textbf{Efficient Reasoning.}\;
As test-time compute for reasoning is scaling, it is important to scale it efficiently. Various methods have been proposed to increase the efficiency of reasoning including length based regularization \citep{kimiteam2025kimik15scalingreinforcement, arora2025training, aggarwal2025l, luo2025o1prunerlengthharmonizingfinetuningo1like} and iterative summarization \citep{aghajohari2025markovianthinkerarchitectureagnosticlinear, yan2025inftythinkbreakinglengthlimits, yang2025pencil}. However, these works primarily target the monolithic CoT setting whereas we focus on the self-parallelized reasoning setting.

\textbf{Parallelizing LLM Threads.}\;
Our work is most closely related to parallelizing reasoning. Although techniques such as Self-Consistency improve performance, all generations are independent and lack coordination. Several papers have proposed using parallel reasoning structures such as in the works by \cite{yao2023tree, pmlr-v235-du24e, ning2024skeletonofthought, teng2025atom}.
Adaptive Parallel Reasoning by \citet{pan2025learning} proposes the \textbf{fork-join} parallelism model and optimizes LLMs via RL for the game of Countdown. Multiverse introduced by \citet{yang2025multiverse} also parallelizes inference whilst also introducing modifications in the attention mechanism. Another body of research \citep{chen2025aspdunlockingadaptiveserialparallel, zheng2025parallelr1parallelthinkingreinforcement, yang2025multiverse, macfarlane2025instilling, lian2025threadweaveradaptivethreadingefficient} focuses on extracting parallelizable reasoning structures from CoTs for math problems. Relatedly, \citet{jin2025learning, ning2024skeletonofthought, liu2024aparllmsautoparallelautoregressive} focus on decoding different parts of a Chain-of-Thought in parallel. However, none of these works focus on an explicit framework to pass messages between independent LLM threads with a persistent context.

\textbf{Multi-Agent Frameworks.} Multi-agent frameworks such as CAMEL \citep{li2023camel}, MetaGPT \citep{hong2024metagpt}, AutoGen \citep{wu2024autogen}, GPTSwarm \citep{pmlr-v235-zhuge24a} and AgentVerse \citep{chen2024agentverse}  study how multiple specialized LLM instances can collaborate through natural-language communication. These systems show that coordination and role specialization can improve performance on complex tasks. However, in most existing frameworks, agent roles, communication patterns, and agent creation mechanisms are largely pre-defined by the developer. In contrast, our work focuses on decomposing a single reasoning task using \emph{runtime}-level communication primitives inspired by MPI, where spawn/send/recv/stop operations are first-class decoding directives which allow task specific agent creation and communication. Crucially, the use of these primitives is learned rather than hard-coded, allowing the model to flexibly allocate computation and more effectively leverage its internal reasoning capabilities.

\section{Discussion, Conclusions and Limitations}
In this work, we introduce MPLMs, which enable LLMs to dynamically create new threads and communicate through point-to-point messages. This capability allows LLMs to break down complex tasks and reason about individual sub-problems in a decentralized fashion, while communicating only when necessary. This resembles how an actual human organization, such as a company, might function. To study the benefits of MPLMs, we instantiate them across both structured reasoning tasks and realistic long-context tasks. On Sudoku and SAT, we show that MPLMs require less context and enable more efficient inference compared to other reasoning paradigms. On LongBench-v2, we further find that sufficiently capable prompted models can already use MPLM directives to perform iterative evidence aggregation more efficiently, reducing latency while maintaining or improving accuracy. 

We also acknowledge several limitations of our work. A key limitation is that, MPLMs require each worker to determine which other workers it should communicate with and when such communication is necessary. In structured domains such as Sudoku and SAT, this communication pattern can be specified relatively clearly, but in open-ended tasks it may require careful prompting or additional training. We believe that an important direction for future work is to generate parallel message-passing CoT data for training stronger MPLM policies beyond predefined scaffolds. This could enable MPLMs to generalize to less structured domains like agentic reasoning, code generation, and theorem proving. We also believe that additional communication primitives, such as broadcasting and all-gather, may further improve the expressiveness and efficiency of MPLM systems.

\section{Acknowledgements}

We would like to thank Yiding Jiang, Yuda Song, Fahim Tajwar, Abitha Thankaraj, Bhavya Agrawalla, Pranjal Aggarwal, Gaurav Ghosal, Jatin Prakash and Lili Chen for their feedback on the project and suggestions on  improving the manuscript. We would also like to thank the high quality feedback from the anonymous reviewers. This work used GH200 GPUs at DeltaAI through allocation CIS251389 and CIS250813 from the Advanced Cyberinfrastructure Coordination Ecosystem: Services \& Support (ACCESS) program, which is supported by U.S. National Science Foundation grants \#2138259, \#2138286, \#2138307, \#2137603, and \#2138296. The authors also thank the CMU FLAME center and the CMU Babel Compute Cluster for additional compute support during the early phases of this project. DA is supported by the National Science Foundation under Grants CCF-2106778 and DMS-2134080, and used compute allocations. GKS is partially supported by a STTR grant. 


\bibliography{reference}

\clearpage
\appendix

\section{MPLM Inference Psuedo-code}
\label{sec:pseudocode}

\begin{algorithm}[h!]
\caption{MPI-style Parallel Decoding (wait-for-all version) }
\label{alg:mpi_vllm}
\begin{algorithmic}[1]

\REQUIRE Inference engine $\mathcal{E}$ (vLLM) that decodes multiple threads in parallel
\REQUIRE Single initial prompt \texttt{prompt} for thread $0$

\STATE $\texttt{Running, Runnable} \leftarrow \{0\}, \{0\}$ \COMMENT{all alive threads}
\STATE $\texttt{Inbox}[0], \;\texttt{WaitingFor}[0] \leftarrow [\ ], [\ ]$

\WHILE{$\texttt{Running} \neq \emptyset$}

  \STATE $\mathcal{E}$ decodes all threads in $\texttt{Runnable}$ in parallel until a directive boundary (or EOS)

  \FORALL[conceptually in parallel]{$i \in \texttt{Runnable}$}
    \IF{thread $i$ reached a directive boundary (or EOS)}
      \STATE Parse one event $e_i$ from thread $i$

      \IF{$e_i$ is \texttt{<spawn [}$\mathcal{J}$\texttt{]>} \texttt{prompt'} \texttt{</spawn>}}
        \FORALL{$j \in \mathcal{J}$}
            \STATE $\texttt{Running, Runnable} \leftarrow \texttt{Running} \cup \{j\}, \texttt{Runnable} \cup \{j\}$
            \STATE $\texttt{Inbox}[j], \;\texttt{WaitingFor}[j] \leftarrow [\ ], [\ ]$
            \STATE Submit \texttt{prompt'} to $\mathcal{E}$ as thread $j$
        \ENDFOR
      \ENDIF

      \IF{$e_i$ is \texttt{<send [}$\mathcal{K}$\texttt{]>} $\texttt{msg}$ \texttt{</send>}}
        \FORALL{$k \in \mathcal{K}$}
            \STATE Append $(i,\texttt{msg})$ to $\texttt{Inbox}[k]$
        \ENDFOR
      \ENDIF

      \IF{$e_i$ is \texttt{<recv [}$\mathcal{R}$\texttt{]>}}
        \STATE $\texttt{WaitingFor}[i] \leftarrow \mathcal{R}$
        \STATE $\texttt{Runnable} \leftarrow \texttt{Runnable} \setminus \{i\}$ \COMMENT{block immediately on recv}
      \ENDIF

      \IF{$e_i$ is \texttt{<stop>} \OR\ EOS}
        \STATE $\texttt{Running, Runnable} \leftarrow \texttt{Running} \setminus \{i\}, \texttt{Runnable} \setminus \{i\}$
        \STATE $\texttt{WaitingFor}[i] \leftarrow [\ ]$
      \ENDIF

    \ENDIF
  \ENDFOR

  \STATE $\texttt{Blocked} \leftarrow \texttt{Running} \setminus \texttt{Runnable}$
  \FORALL[conceptually in parallel]{$i \in \texttt{Blocked}$}
  \IF{$\texttt{WaitingFor}[i] = [\ ]$}
    \STATE $\texttt{Runnable} \leftarrow \texttt{Runnable} \cup \{i\}$ \COMMENT{all required senders died}
  \ELSE
    \STATE $\texttt{HaveSenders} \leftarrow \{\, s \mid \exists (s,m)\in \texttt{Inbox}[i] \,\}$
    \IF{$\texttt{WaitingFor}[i] \subseteq \texttt{HaveSenders}$}
      \STATE Consume the needed messages from $\texttt{Inbox}[i]$ for all $s \in \texttt{WaitingFor}[i]$
      \STATE $\texttt{WaitingFor}[i] \leftarrow [\ ]$
      \STATE $\texttt{Runnable} \leftarrow \texttt{Runnable} \cup \{i\}$
    \ENDIF
    \ENDIF
  \ENDFOR

\ENDWHILE

\end{algorithmic}
\end{algorithm}

\newpage

\section{Additional Aspects of MPLMs}
\label{app:mplm_additional}
\subsection{Implementation of \texttt{recv}}
\label{sec:sync}
The \texttt{recv} directive requires inference for a thread to be stopped to ensure all incoming messages are accumulated. We note that \texttt{recv} can be implemented in various ways. The implementations we consider here are:
\begin{enumerate}
\item \textbf{Asynchronous (wait-for-all):} When a thread executes \texttt{recv(S)}, it blocks until it has received all required messages from all active threads in the set $S$ . Once every required sender has produced a message for that worker, the thread becomes runnable immediately. This is the main version we consider in the work.
\item \textbf{Asynchronous (wait-for-any):} When a thread executes \texttt{recv(S)}, it blocks until it receives at least one message from any sender in $S$. The runtime then immediately unblocks the thread, allowing it to continue decoding with partial information.
\end{enumerate}
These two implementations can have different use cases which could potentially benefit in different types of applications (see Appendix \ref{sec:sync_under_variable_latency} for an example on a toy setup). It is also possible to have different directives for each of these implementations, but we leave this for future work.

\subsection{Preemption via Early Subtree Cancellation}
In addition to reducing aggregation cost, MPLMs support \emph{preemption}. Here, we use \emph{preemption} to mean early cancellation of active descendants once they can no longer affect the final output. This behavior is natural in MPLMs because the runtime maintains the spawn tree and thread identities are globally addressable, so messages can be sent not only between siblings but also across different levels of the threads.

Consider a thread $p$ that spawns a set of children $C(p)$ to solve a hierarchical sub-problem. In many reasoning tasks, local decision of $p$ is determined by a predicate over incoming child messages that may become satisfied before all children finish. For example, in OR-type search problems, a single positive certificate may already suffice. Once $p$ has received enough information to determine its output, it can immediately send a message to its own parent (or any relevant ancestor) and terminate. The runtime then cancels the still-active descendants in the subtree rooted at $p$ whose outputs are no longer reachable from a live consumer.

This capability is difficult to express in standard fork-join execution, where progress is tied to join directives, and differs from controller-based tree search, where branch pruning is externalized to a centralized controller. Importantly, MPLM preemption does not require introducing a new decoding primitive; it emerges from persistent thread identities, direct message passing, and runtime cancellation of unreachable descendants.

\subsection{Context Management via Respawning}
The \texttt{spawn} directive can be used in a special manner where a thread can spawn a thread that is already existing, including itself. Note that in this case, the controller merely creates a new thread with the same logical identity but with a different prompt. We term this capability: \textit{respawning}. This capability can be used by the LLM in interesting ways. For instance, a worker can construct a compact \emph{inheritance payload} that summarizes the minimal information required for future progress (e.g., verified intermediate results, current tasks or constraints) and then respawn. By discarding the redundant history, the newly spawned worker begins from this curated and compressed context, effectively resetting its working memory without losing task-critical state.

This is highly beneficial because an MPLM can summarize its own context resulting in: \textbf{(i) Bounded per-step cost.} By controlling the context length of each worker, respawning prevents per-token computation from becoming unbounded over time, reducing wall-clock latency for long-running tasks. \textbf{(ii) Unbounded reasoning.} Repeatedly summarizing each worker's state and continuing from a refreshed context could, in theory, allow threads to run for arbitrarily many iterations without exceeding the model's context window.

\section{Sudoku Generation Procedure}
\label{app:data_gen}

For generating Sudoku puzzles, we use a deletion probability $p$ which is a measure of the \textit{sparsity} of the puzzle. The higher the value of $p$, the sparser it will be and in general would require more iterations to solve using the naked-singles strategy. In our work, we use $p=1.0$ for both training and inference.

\begin{algorithm*}[h]
\caption{Sudoku Dataset Generation (Naked-Single Solvable)}
\label{alg:sudoku_gen_compact}
\begin{algorithmic}[1]

\REQUIRE Block size $\texttt{base\_n}$, grid size $N \leftarrow \texttt{base\_n}^2$
\REQUIRE Target size $\texttt{size}$, deletion probability $p \in [0,1]$

\STATE $\mathcal{D} \leftarrow [\ ]$
\WHILE{$|\mathcal{D}| < \texttt{size}$}
  \STATE $\texttt{solution} \leftarrow$ \textsc{RandomSolvedGrid}$(N,\texttt{base\_n})$
  \STATE $\texttt{original} \leftarrow \texttt{solution}$
  \STATE $\texttt{cells} \leftarrow$ all cell indices $(r,c)$
  \FORALL{$(r,c) \in \texttt{cells}$}
    \STATE $\texttt{saved} \leftarrow \texttt{original}[r][c]$;\ \ $\texttt{original}[r][c] \leftarrow 0$
    \IF{\textsc{NakedSingleSolvable}$(\texttt{original}, \texttt{solution}, \texttt{base\_n})$}
      \STATE \COMMENT{with probability p remove this cell, otherwise restore it}
      \IF{$\textsc{Bernoulli}(p)=0$} 
        \STATE $\texttt{original}[r][c] \leftarrow \texttt{saved}$
      \ENDIF
    \ELSE
      \STATE $\texttt{original}[r][c] \leftarrow \texttt{saved}$
    \ENDIF
  \ENDFOR
  \STATE Append $(\texttt{original}, \texttt{solution})$ to $\mathcal{D}$
\ENDWHILE
\STATE \textbf{return} $\mathcal{D}$

\vspace{0.4em}
\STATE
\STATE \textbf{Predicate} \textsc{NakedSingleSolvable}$(\texttt{original}, \texttt{solution}, \texttt{base\_n})$
\STATE $N \leftarrow \texttt{base\_n}^2$;\ \ $\texttt{board} \leftarrow$ copy of $\texttt{original}$
\WHILE{there exists an empty cell in $\texttt{board}$ that has exactly one feasible value}
  \STATE Pick such a cell $(r,c)$
  \STATE Compute $\texttt{used} \leftarrow$ numbers already present in row $r$, column $c$, and the $\texttt{base\_n}\times\texttt{base\_n}$ box of $(r,c)$
  \STATE Let $\texttt{cand} \leftarrow \{1,\dots,N\}\setminus \texttt{used}$
  \IF{$|\texttt{cand}| \neq 1$}
    \STATE Continue
  \ENDIF
  \STATE Fill $\texttt{board}[r][c]$ with the unique value in $\texttt{cand}$
\ENDWHILE
\STATE \textbf{return} $(\texttt{board} = \texttt{solution})$

\end{algorithmic}
\end{algorithm*}

\subsection{Training details}
We generate CoT for each of the training paradigms using a simple Python program. An example of the type of CoTs we use is given in \cref{app:sample_cot,app:sample_cot_fj,app:sample_cot_serial}. This gives us a dataset of (prompt, response) pairs which we use for SFT using the LlamaFactory \citep{zheng2024llamafactory} library. We use a batch size of 128 with a learning rate of $10^{-5}$ using the Adam \citep{DBLP:journals/corr/KingmaB14} optimizer. We observe that as we scale the size of the Sudoku puzzle, the master traces in the CoT become less populous which causes lower learning signal on them and the model fails to perform the work required by the master. In order to counteract this, we upweight the number of master samples in the dataset by a factor of 5 for $N=3$, a factor of 20 for $N=4$ and 100 for $N=5$. We choose the model based on the validation token accuracy on a subset of 5k prompts. 

\section{SAT Training Details}\label{app:sat_accuracy}

For SAT, satisfiable instances are generated by first sampling a random assignment and then constructing clauses that are consistent with it, whereas unsatisfiable instances are produced by random sampling and verified using a symbolic SAT solver. Then, we generate CoT for each of the training paradigms using a simple Python program. An example of the type of CoTs we use is given in \cref{app:sample_cot_sat_mplm,app:sample_cot_sat_fj,app:sample_cot_sat_serial}. This gives us a dataset of (prompt, response) pairs which we use for SFT using the LlamaFactory \citep{zheng2024llamafactory} library. We use a batch size of 128 with a learning rate of $10^{-5}$ using the Adam \citep{DBLP:journals/corr/KingmaB14} optimizer. We construct a training set of 10{,}000 instances by sampling from problems with 8 to 20 variables. To better support performance on harder instances, we upsample problems with larger numbers of variables during training. We then evaluate the resulting model on test sets consisting of 100 instances for each variable count from 8 to 20. The SAT accuracy results are reported below.

\begin{table}[H]
\centering
\scriptsize
\resizebox{\linewidth}{!}{%
\begin{tabular}{lccccccccccccc}
\hline
\textbf{Model} & \textbf{8} & \textbf{9} & \textbf{10} & \textbf{11} & \textbf{12} & \textbf{13} & \textbf{14} & \textbf{15} & \textbf{16} & \textbf{17} & \textbf{18} & \textbf{19} & \textbf{20} \\
\hline
MPLM   & 100\% & 91\% & 99\% & 91\% & 95\% & 94\% & 92\% & 87\% & 92\% & 89\% & 90\% & 86\% & 85\% \\
FJ     & 97\%  & 95\% & 94\% & 93\% & 91\% & 93\% & 90\% & 91\% & 91\% & 88\% & 90\% & 85\% & 86\% \\
Serial & 100\% & 95\% & 98\% & 93\% & 94\% & \textcolor{red}{\xmark} & \textcolor{red}{\xmark} & \textcolor{red}{\xmark} & \textcolor{red}{\xmark} & \textcolor{red}{\xmark} & \textcolor{red}{\xmark} & \textcolor{red}{\xmark} & \textcolor{red}{\xmark} \\
\hline
\end{tabular}%
}
\caption{Accuracy on SAT problems per variable count (8--20). A red cross (\textcolor{red}{\xmark}) denotes infeasible evaluation due to context length constraints.}
\label{tab:accuracy}
\vspace{-0.5em}
\end{table}

\section{Sample Sudoku Chain-of-Thoughts (MPLM)}
\label{app:sample_cot}

This section provides representative traces from the master thread and a worker thread.
We display the raw directive syntax used by the runtime, including \texttt{<spawn>}, \texttt{<recv>}, \texttt{<send>}, and \texttt{<stop>}.
For readability, we truncate long candidate maps.

\vspace{0.5em}
\subsection{Master Trace (Initialization)}
\noindent\textbf{Master prompt.}
\begin{small}
\begin{verbatim}
Solve this Sudoku:
{ (0,0): {6}, (0,1): {8}, (0,2): {9}, (0,3): {4}, (0,4): {3}, (0,5): {1},
  (0,6): {2}, (0,8): {7}, (1,0): {1}, (1,1): {4}, (1,2): {3}, ...,
  (8,8): {4} }
\end{verbatim}
\end{small}

\noindent\textbf{Master response.}
\begin{small}
\begin{verbatim}
<spawn [1,2,3,4,5,6,7,8,9,10,11,12,13,14,15,16,17,18,19,20,21,22,23,24,25,26,
        27,28,29,30,31,32,33,34,35,36,37,38,39,40,41,42,43,44,45,46,47,48,49,
        50,51,52,53,54,55,56,57,58,59,60,61,62,63,64,65,66,67,68,69,70,71,72,
        73,74,75,76,77,78,79,80,81]>
{ (0,0): {6}, (0,1): {8}, (0,2): {9}, (0,3): {4}, (0,4): {3}, (0,5): {1},
  (0,6): {2}, (0,8): {7}, (1,0): {1}, (1,1): {4}, (1,2): {3}, ...,
  (8,8): {4} }
</spawn>
<recv [1,2,3,4,5,6,7,8,9,10,11,12,13,14,15,16,17,18,19,20,21,22,23,24,25,26,
       27,28,29,30,31,32,33,34,35,36,37,38,39,40,41,42,43,44,45,46,47,48,49,
       50,51,52,53,54,55,56,57,58,59,60,61,62,63,64,65,66,67,68,69,70,71,72,
       73,74,75,76,77,78,79,80,81]>
\end{verbatim}
\end{small}

\vspace{0.8em}
\subsection{Worker Trace (Example: Worker 42)}
\noindent\textbf{Worker prompt (first turn).}
\begin{small}
\begin{verbatim}
Your id is: 42.
{ (0,0): {6}, (0,1): {8}, (0,2): {9}, (0,3): {4}, (0,4): {3}, (0,5): {1},
  (0,6): {2}, (0,8): {7}, (1,0): {1}, (1,1): {4}, (1,2): {3}, ...,
  (8,8): {4} }
\end{verbatim}
\end{small}

\noindent\textbf{Worker response (first turn).}
\begin{small}
\begin{verbatim}
My id is 42 which corresponds to cell (4,5). The cell is empty.
The remaining candidates are: {1, 2, 3, 4, 5, 6, 7, 8, 9}.
<recv>
\end{verbatim}
\end{small}

\vspace{0.6em}
\noindent\textbf{Worker prompt (after receiving neighbor messages).}
\begin{small}
\begin{verbatim}
Your id is: 42.
{ (0,0): {6}, (0,1): {8}, (0,2): {9}, (0,3): {4}, (0,4): {3}, (0,5): {1},
  (0,6): {2}, (0,8): {7}, (1,0): {1}, (1,1): {4}, (1,2): {3}, ...,
  (8,8): {4} }
My id is 42 which corresponds to cell (4,5). The cell is empty.
The remaining candidates are: {1, 2, 3, 4, 5, 6, 7, 8, 9}.
<recv>
{ From 6: 1 }
{ From 31: 7 }
{ From 32: 2 }
...
{ From 78: 8 }
</recv>
\end{verbatim}
\end{small}

\noindent\textbf{Worker response (computes its value and reports).}
\begin{small}
\begin{verbatim}
The neighboring cells have sent the following values: {1, 2, 3, 4, 6, 7, 8, 9}.
The remaining candidates are: {5}.
<send [0,6,15,24,31,32,33,37,38,39,40,41,43,44,45,49,50,51,60,69,78]>
{ From 42: 5 }
</send>
<stop>
\end{verbatim}
\end{small}

\vspace{0.8em}
\subsection{Master Trace (Termination)}
\noindent\textbf{Master prompt (after collecting worker outputs).}
\begin{small}
\begin{verbatim}
Solve this Sudoku:
{ (0,0): {6}, (0,1): {8}, (0,2): {9}, (0,3): {4}, (0,4): {3}, (0,5): {1},
  (0,6): {2}, (0,8): {7}, (1,0): {1}, (1,1): {4}, (1,2): {3}, ...,
  (8,8): {4} }
<spawn [1,2,3,4,5,6,7,8,9,10,11,12,13,14,15,16,17,18,19,20,21,22,23,24,25,26,
        27,28,29,30,31,32,33,34,35,36,37,38,39,40,41,42,43,44,45,46,47,48,49,
        50,51,52,53,54,55,56,57,58,59,60,61,62,63,64,65,66,67,68,69,70,71,72,
        73,74,75,76,77,78,79,80,81]>{ ... }</spawn>
<recv [1,2,3,4,5,6,7,8,9,10,11,12,13,14,15,16,17,18,19,20,21,22,23,24,25,26,
       27,28,29,30,31,32,33,34,35,36,37,38,39,40,41,42,43,44,45,46,47,48,49,
       50,51,52,53,54,55,56,57,58,59,60,61,62,63,64,65,66,67,68,69,70,71,72,
       73,74,75,76,77,78,79,80,81]>
{ From 1: 6 }
{ From 2: 8 }
{ From 3: 9 }
...
{ From 81: 4 }
</recv>
\end{verbatim}
\end{small}

\noindent\textbf{Master response (final).}
\begin{small}
\begin{verbatim}
<stop>
[[6, 8, 9, 4, 3, 1, 2, 5, 7],
 [1, 4, 3, 2, 5, 7, 8, 9, 6],
 [7, 2, 5, 8, 9, 6, 4, 3, 1],
 [5, 1, 4, 7, 2, 9, 6, 8, 3],
 [3, 6, 8, 1, 4, 5, 7, 2, 9],
 [9, 7, 2, 6, 8, 3, 1, 4, 5],
 [4, 3, 6, 5, 1, 2, 9, 7, 8],
 [8, 9, 7, 3, 6, 4, 5, 1, 2],
 [2, 5, 1, 9, 7, 8, 3, 6, 4]]
</stop>
\end{verbatim}
\end{small}

\section{Sample Sudoku Chain-of-Thoughts (Fork-Join)}
\label{app:sample_cot_fj}

This section provides representative traces from the master thread and a worker thread
under the Fork/Join (FJ) scheduling setting.
We display the raw directive syntax used by the runtime, including \texttt{<spawn>}, \texttt{<recv>}, \texttt{<send>}, and \texttt{<stop>}.
For readability, we truncate long candidate maps.

\vspace{0.5em}
\subsection{Master Trace (Initialization)}
\noindent\textbf{Master prompt.}
\begin{small}
\begin{verbatim}
Solve this Sudoku:
{ (0,0): {6}, (0,5): {1}, (0,7): {5}, (1,2): {3}, (1,3): {2},
  (2,0): {7}, (2,1): {2}, (2,2): {5}, (2,6): {4}, (3,2): {4},
  (3,6): {6}, (3,8): {3}, (4,7): {2}, (5,0): {9}, (5,1): {7},
  (5,3): {6}, (5,8): {5}, (6,1): {3}, (7,0): {8}, (7,3): {3},
  (7,4): {6}, (8,2): {1}, (8,3): {9}, (8,7): {6}, (8,8): {4} }
\end{verbatim}
\end{small}

\noindent\textbf{Master response.}
\begin{small}
\begin{verbatim}
<spawn [2,3,4,5,7,9,10,11,14,15,16,17,18,22,23,24,26,27,28,29,
        31,32,33,35,37,38,39,40,41,42,43,45,48,50,51,52,53,55,
        57,58,59,60,61,62,63,65,66,69,70,71,72,73,74,77,78,79]>
{ (0,0): {6}, (0,5): {1}, (0,7): {5}, (1,2): {3}, (1,3): {2},
  (2,0): {7}, (2,1): {2}, (2,2): {5}, (2,6): {4}, (3,2): {4},
  (3,6): {6}, (3,8): {3}, (4,7): {2}, (5,0): {9}, (5,1): {7},
  (5,3): {6}, (5,8): {5}, (6,1): {3}, (7,0): {8}, (7,3): {3},
  (7,4): {6}, (8,2): {1}, (8,3): {9}, (8,7): {6}, (8,8): {4} }
</spawn>
<recv [2,3,4,5,7,9,10,11,14,15,16,17,18,22,23,24,26,27,28,29,
       31,32,33,35,37,38,39,40,41,42,43,45,48,50,51,52,53,55,
       57,58,59,60,61,62,63,65,66,69,70,71,72,73,74,77,78,79]>
\end{verbatim}
\end{small}

\vspace{0.6em}
\noindent\textbf{Master prompt (after receiving worker reports).}
\begin{small}
\begin{verbatim}
Solve this Sudoku:
{ (0,0): {6}, (0,5): {1}, (0,7): {5}, (1,2): {3}, (1,3): {2},
  (2,0): {7}, (2,1): {2}, (2,2): {5}, (2,6): {4}, (3,2): {4},
  (3,6): {6}, (3,8): {3}, (4,7): {2}, (5,0): {9}, (5,1): {7},
  (5,3): {6}, (5,8): {5}, (6,1): {3}, (7,0): {8}, (7,3): {3},
  (7,4): {6}, (8,2): {1}, (8,3): {9}, (8,7): {6}, (8,8): {4} }
<spawn [2,3,4,5,7,9,10,11,14,15,16,17,18,22,23,24,26,27,28,29,
        31,32,33,35,37,38,39,40,41,42,43,45,48,50,51,52,53,55,
        57,58,59,60,61,62,63,65,66,69,70,71,72,73,74,77,78,79]>{ ... }</spawn>
<recv [2,3,4,5,7,9,10,11,14,15,16,17,18,22,23,24,26,27,28,29,
       31,32,33,35,37,38,39,40,41,42,43,45,48,50,51,52,53,55,
       57,58,59,60,61,62,63,65,66,69,70,71,72,73,74,77,78,79]>
{ From 2: None }
{ From 3: None }
{ From 4: None }
{ From 5: None }
...
{ From 22: 8 }
...
{ From 74: 5 }
...
{ From 79: None }
</recv>
\end{verbatim}
\end{small}

\vspace{0.5em}
\noindent\textbf{Master response (next fork/join step).}
\begin{small}
\begin{verbatim}
<spawn [2,3,4,5,7,9,10,11,14,15,16,17,18,23,24,26,27,28,29,31,
        32,33,35,37,38,39,40,41,42,43,45,48,50,51,52,53,55,57,
        58,59,60,61,62,63,65,66,69,70,71,72,73,77,78,79]>
{ (0,0): {6}, (0,5): {1}, (0,7): {5}, (1,2): {3}, (1,3): {2},
  (2,0): {7}, (2,1): {2}, (2,2): {5}, (2,3): {8}, (2,6): {4},
  (3,2): {4}, (3,6): {6}, (3,8): {3}, (4,7): {2}, (5,0): {9},
  (5,1): {7}, (5,3): {6}, (5,8): {5}, (6,1): {3}, (7,0): {8},
  (7,3): {3}, (7,4): {6}, (8,1): {5}, (8,2): {1}, (8,3): {9},
  (8,7): {6}, (8,8): {4} }
</spawn>
<recv [2,3,4,5,7,9,10,11,14,15,16,17,18,23,24,26,27,28,29,31,
       32,33,35,37,38,39,40,41,42,43,45,48,50,51,52,53,55,57,
       58,59,60,61,62,63,65,66,69,70,71,72,73,77,78,79]>
\end{verbatim}
\end{small}

\vspace{0.8em}
\subsection{Worker Trace (Example: Worker 42)}
\noindent\textbf{Worker prompt (first turn).}
\begin{small}
\begin{verbatim}
Your id is: 42.
{ (0,0): {6}, (0,5): {1}, (0,7): {5}, (1,2): {3}, (1,3): {2},
  (2,0): {7}, (2,1): {2}, (2,2): {5}, (2,6): {4}, (3,2): {4},
  (3,6): {6}, (3,8): {3}, (4,7): {2}, (5,0): {9}, (5,1): {7},
  (5,3): {6}, (5,8): {5}, (6,1): {3}, (7,0): {8}, (7,3): {3},
  (7,4): {6}, (8,2): {1}, (8,3): {9}, (8,7): {6}, (8,8): {4} }
\end{verbatim}
\end{small}

\noindent\textbf{Worker response (first turn).}
\begin{small}
\begin{verbatim}
My id is 42 which corresponds to cell (4,5).
The neighboring cells already have the following values: {1, 2, 6}.
The remaining candidates are: {3, 4, 5, 7, 8, 9}.
<send [0]>{ From 42: None }</send><stop>
\end{verbatim}
\end{small}

\vspace{0.6em}
\noindent\textbf{Worker prompt (later fork; after more iterations).}
\begin{small}
\begin{verbatim}
Your id is: 42.
{ (0,0): {6}, (0,2): {9}, (0,3): {4}, (0,4): {3}, (0,5): {1},
  (0,6): {2}, (0,7): {5}, (0,8): {7}, (1,0): {1}, (1,1): {4},
  (1,2): {3}, (1,3): {2}, (2,0): {7}, (2,1): {2}, (2,2): {5},
  (2,3): {8}, (2,4): {9}, (2,5): {6}, (2,6): {4}, (2,7): {3},
  (2,8): {1}, (3,0): {5}, (3,1): {1}, (3,2): {4}, (3,3): {7},
  (3,6): {6}, (3,8): {3}, (4,0): {3}, (4,1): {6}, (4,2): {8},
  (4,7): {2}, (4,8): {9}, ...,
  (8,7): {6}, (8,8): {4} }
\end{verbatim}
\end{small}

\noindent\textbf{Worker response (computes its value and reports).}
\begin{small}
\begin{verbatim}
My id is 42 which corresponds to cell (4,5).
The neighboring cells already have the following values: {1, 2, 3, 4, 6, 7, 8, 9}.
The remaining candidates are: {5}.
<send [0]>{ From 42: 5 }</send><stop>
\end{verbatim}
\end{small}

\vspace{0.8em}
\subsection{Master Trace (Termination)}
\noindent\textbf{Master response (final stop).}
\begin{small}
\begin{verbatim}
<stop>
[[6, 8, 9, 4, 3, 1, 2, 5, 7],
 [1, 4, 3, 2, 5, 7, 8, 9, 6],
 [7, 2, 5, 8, 9, 6, 4, 3, 1],
 [5, 1, 4, 7, 2, 9, 6, 8, 3],
 [3, 6, 8, 1, 4, 5, 7, 2, 9],
 [9, 7, 2, 6, 8, 3, 1, 4, 5],
 [4, 3, 6, 5, 1, 2, 9, 7, 8],
 [8, 9, 7, 3, 6, 4, 5, 1, 2],
 [2, 5, 1, 9, 7, 8, 3, 6, 4]]
</stop>
\end{verbatim}
\end{small}

\newpage
\section{Sample Sudoku Chain-of-Thoughts (Serial)}
\label{app:sample_cot_serial}

This section provides a representative trace for the Serial baseline.
Unlike the message-passing settings, Serial performs iterative constraint propagation
in a single thread (round by round) and terminates with \texttt{<stop>}.
For readability, we truncate long candidate maps.

\vspace{0.5em}
\subsection{Serial Trace (Rounds)}
\noindent\textbf{Serial prompt.}
\begin{small}
\begin{verbatim}
Solve this Sudoku:
{ (0,0): {6}, (0,5): {1}, (0,7): {5}, (1,2): {3}, (1,3): {2},
  (2,0): {7}, (2,1): {2}, (2,2): {5}, (2,6): {4}, (3,2): {4},
  (3,6): {6}, (3,8): {3}, (4,7): {2}, (5,0): {9}, (5,1): {7},
  (5,3): {6}, (5,8): {5}, (6,1): {3}, (7,0): {8}, (7,3): {3},
  (7,4): {6}, (8,2): {1}, (8,3): {9}, (8,7): {6}, (8,8): {4} }
\end{verbatim}
\end{small}

\noindent\textbf{Serial response}
\begin{small}
\begin{verbatim}

Round 1:
For cell (0,1), the neighboring cells have the values {1,2,3,5,6,7}.
The remaining candidates are: {4,8,9}
For cell (0,2), the neighboring cells have the values {1,2,3,4,5,6,7}.
The remaining candidates are: {8,9}
For cell (0,3), the neighboring cells have the values {1,2,3,5,6,9}.
The remaining candidates are: {4,7,8}
...
For cell (8,1), the neighboring cells have the values {1,2,3,4,6,7,8,9}.
The remaining candidates are: {5}
...
The current remaining candidates of the sudoku is :
{ (0,0): {6}, (0,1): {4,8,9}, (0,2): {8,9}, (0,3): {4,7,8}, ... }

...

Round 21:
For cell (5,5), the neighboring cells have the values {1,2,4,5,6,7,8,9}.
The remaining candidates are: {3}
For cell (6,4), the neighboring cells have the values {2,3,4,5,6,7,8,9}.
The remaining candidates are: {1}
For cell (6,6), the neighboring cells have the values {1,2,3,4,5,6,7,8}.
The remaining candidates are: {9}
For cell (8,4), the neighboring cells have the values {1,2,3,4,5,6,8,9}.
The remaining candidates are: {7}
For cell (8,6), the neighboring cells have the values {1,2,4,5,6,7,8,9}.
The remaining candidates are: {3}
The current remaining candidates of the sudoku is :
{ (0,0): {6}, (0,1): {8}, (0,2): {9}, (0,3): {4}, (0,4): {3}, ...,
  (8,6): {3}, (8,7): {6}, (8,8): {4} }
<stop>
[[6, 8, 9, 4, 3, 1, 2, 5, 7],
 [1, 4, 3, 2, 5, 7, 8, 9, 6],
 [7, 2, 5, 8, 9, 6, 4, 3, 1],
 [5, 1, 4, 7, 2, 9, 6, 8, 3],
 [3, 6, 8, 1, 4, 5, 7, 2, 9],
 [9, 7, 2, 6, 8, 3, 1, 4, 5],
 [4, 3, 6, 5, 1, 2, 9, 7, 8],
 [8, 9, 7, 3, 6, 4, 5, 1, 2],
 [2, 5, 1, 9, 7, 8, 3, 6, 4]]
</stop>
\end{verbatim}
\end{small}

\newpage
\section{Sample 3-SAT Chain-of-Thoughts (MPLM)}
\label{app:sample_cot_sat_mplm}

This section provides representative traces from the root agent and a worker
agent on a 5-variable 3-SAT instance (answer: SAT).
The root branches on $x_1$; each child branches on $x_2$; each grandchild is a
leaf solved by unit propagation.
In MPLM, each parent performs a \emph{asynchronous receive}: it reads child results one
at a time and stops as soon as the subproblem is determined.
We display the raw directive syntax including
\texttt{<spawn>}, \texttt{<recv>}, \texttt{<send>}, and \texttt{<stop>}.

\vspace{0.5em}
\subsection{Root Agent Trace}

\noindent\textbf{Root prompt.}
\begin{small}
\begin{verbatim}
( 2 v -1 v 4 ) ^ ( 4 v 5 v -2 ) ^ ( -2 v 3 v 1 ) ^ ( -5 v -2 v 4 ) ^
( 4 v 3 v -1 ) ^ ( -5 v 1 v -4 ) ^ ( 4 v -2 v 3 ) ^ ( -5 v 2 v -3 ) ^
( -5 v 2 v 1 ) ^ ( -5 v -1 v -3 ) ^ ( 4 v -2 v -1 ) ^ ( 4 v 5 v -2 ) ^
( 1 v 3 v 2 ) ^ ( -2 v 4 v 1 ) ^ ( 1 v -5 v -4 ) ^ ( 2 v -1 v -3 ) ^
( -1 v -5 v -4 ) ^ ( -1 v 4 v -5 ) ^ ( -4 v -5 v 3 ) ^ ( -1 v 4 v -2 ) ^
( -2 v 4 v 3 )
\end{verbatim}
\end{small}

\noindent\textbf{Root response (spawns two workers, waits for first result).}
\begin{small}
\begin{verbatim}
Try unit propagation.
No unit clauses.
Branch on x1.
<spawn [0.0]>
( 4 v 5 v -2 ) ^ ( -2 v 3 ) ^ ( -5 v -2 v 4 ) ^ ( -5 v -4 ) ^
( 4 v -2 v 3 ) ^ ( -5 v 2 v -3 ) ^ ( -5 v 2 ) ^ ( 4 v 5 v -2 ) ^
( 3 v 2 ) ^ ( -2 v 4 ) ^ ( -5 v -4 ) ^ ( -4 v -5 v 3 ) ^ ( -2 v 4 v 3 )
</spawn>
<spawn [0.1]>
( 2 v 4 ) ^ ( 4 v 5 v -2 ) ^ ( -5 v -2 v 4 ) ^ ( 4 v 3 ) ^
( 4 v -2 v 3 ) ^ ( -5 v 2 v -3 ) ^ ( -5 v -3 ) ^ ( 4 v -2 ) ^
( 4 v 5 v -2 ) ^ ( 2 v -3 ) ^ ( -5 v -4 ) ^ ( 4 v -5 ) ^
( -4 v -5 v 3 ) ^ ( 4 v -2 ) ^ ( -2 v 4 v 3 )
</spawn>
<recv>
\end{verbatim}
\end{small}

\vspace{0.6em}
\noindent\textbf{Root prompt (after receiving SAT from worker~0.0; worker~0.1 is preempted).}
\begin{small}
\begin{verbatim}
( 2 v -1 v 4 ) ^ ... ^ ( -2 v 4 v 3 )
Try unit propagation.
No unit clauses.
Branch on x1.
<spawn [0.0]>
...
</spawn>
<spawn [0.1]>
...
</spawn>
<recv>
{SAT.}     % only 1 of 2 child results received; worker 0.1 is preempted
</recv>
\end{verbatim}
\end{small}

\noindent\textbf{Root response.}
\begin{small}
\begin{verbatim}
<stop>SAT.</stop>
\end{verbatim}
\end{small}

\vspace{0.8em}
\subsection{Worker Agent Trace (Example: Worker~0.1)}

\noindent\textbf{Worker 0.1 prompt (first turn).}
\begin{small}
\begin{verbatim}
Your id is: 0.1.
( 2 v 4 ) ^ ( 4 v 5 v -2 ) ^ ( -5 v -2 v 4 ) ^ ( 4 v 3 ) ^
( 4 v -2 v 3 ) ^ ( -5 v 2 v -3 ) ^ ( -5 v -3 ) ^ ( 4 v -2 ) ^
( 4 v 5 v -2 ) ^ ( 2 v -3 ) ^ ( -5 v -4 ) ^ ( 4 v -5 ) ^
( -4 v -5 v 3 ) ^ ( 4 v -2 ) ^ ( -2 v 4 v 3 )
\end{verbatim}
\end{small}

\noindent\textbf{Worker 0.1 response (spawns subworkers, waits for first result).}
\begin{small}
\begin{verbatim}
Try unit propagation.
No unit clauses.
Branch on x2.
<spawn [0.1.0]>
( 4 ) ^ ( 4 v 3 ) ^ ( -5 v -3 ) ^ ( -5 v -3 ) ^ ( -3 ) ^ ( -5 v -4 ) ^
( 4 v -5 ) ^ ( -4 v -5 v 3 )
</spawn>
<spawn [0.1.1]>
( 4 v 5 ) ^ ( -5 v 4 ) ^ ( 4 v 3 ) ^ ( 4 v 3 ) ^ ( -5 v -3 ) ^ ( 4 ) ^
( 4 v 5 ) ^ ( -5 v -4 ) ^ ( 4 v -5 ) ^ ( -4 v -5 v 3 ) ^ ( 4 ) ^ ( 4 v 3 )
</spawn>
<recv>
\end{verbatim}
\end{small}

\vspace{0.6em}
\noindent\textbf{Worker 0.1 prompt (after receiving SAT from worker~0.1.0; worker~0.1.1 is preempted).}
\begin{small}
\begin{verbatim}
Your id is: 0.1.
( 2 v 4 ) ^ ... ^ ( -2 v 4 v 3 )
Try unit propagation.
No unit clauses.
Branch on x2.
<spawn [0.1.0]>
...
</spawn>
<spawn [0.1.1]>
...
</spawn>
<recv>
{SAT.}     % only 1 of 2 child results received; worker 0.1.1 is preempted
</recv>
\end{verbatim}
\end{small}

\noindent\textbf{Worker 0.1 response (concludes and reports to root).}
\begin{small}
\begin{verbatim}
Conclusion for this subproblem: SAT.
<send [0]>{SAT.}</send> <stop>
\end{verbatim}
\end{small}

\section{Sample 3-SAT Chain-of-Thoughts (FJ)}
\label{app:sample_cot_sat_fj}

This section shows the same 5-variable instance solved by Fork-Join.
In FJ, the \texttt{<recv [...]>} tag lists \emph{all} expected senders; the
parent blocks until every child result has arrived before aggregating.
For readability we truncate the CNF clauses in the continuation prompts.

\vspace{0.5em}
\subsection{Root Agent Trace}

\noindent\textbf{Root response (spawns two workers, waits for \emph{both}).}
\begin{small}
\begin{verbatim}
Try unit propagation.
No unit clauses.
Branch on x1.
<spawn [0.0]>
( 4 v 5 v -2 ) ^ ( -2 v 3 ) ^ ( -5 v -2 v 4 ) ^ ( -5 v -4 ) ^
( 4 v -2 v 3 ) ^ ( -5 v 2 v -3 ) ^ ( -5 v 2 ) ^ ( 4 v 5 v -2 ) ^
( 3 v 2 ) ^ ( -2 v 4 ) ^ ( -5 v -4 ) ^ ( -4 v -5 v 3 ) ^ ( -2 v 4 v 3 )
</spawn>
<spawn [0.1]>
( 2 v 4 ) ^ ( 4 v 5 v -2 ) ^ ( -5 v -2 v 4 ) ^ ( 4 v 3 ) ^
( 4 v -2 v 3 ) ^ ( -5 v 2 v -3 ) ^ ( -5 v -3 ) ^ ( 4 v -2 ) ^
( 4 v 5 v -2 ) ^ ( 2 v -3 ) ^ ( -5 v -4 ) ^ ( 4 v -5 ) ^
( -4 v -5 v 3 ) ^ ( 4 v -2 ) ^ ( -2 v 4 v 3 )
</spawn>
<recv [0.0, 0.1]>          % waits for ALL children before proceeding
\end{verbatim}
\end{small}

\vspace{0.6em}
\noindent\textbf{Root prompt (after receiving SAT from \emph{both} workers~0.0 and~0.1).}
\begin{small}
\begin{verbatim}
( 2 v -1 v 4 ) ^ ... ^ ( -2 v 4 v 3 )
Try unit propagation.
No unit clauses.
Branch on x1.
<spawn [0.0]>
...
</spawn>
<spawn [0.1]>
...
</spawn>
<recv [0.0, 0.1]>
{SAT.}
{SAT.}     % both child results received before aggregating
</recv>
\end{verbatim}
\end{small}

\noindent\textbf{Root response.}
\begin{small}
\begin{verbatim}
<stop>SAT.</stop>
\end{verbatim}
\end{small}

\vspace{0.8em}
\subsection{Worker Agent Trace (Example: Worker~0.1)}

\noindent\textbf{Worker 0.1 response (spawns subworkers, waits for \emph{both}).}
\begin{small}
\begin{verbatim}
Try unit propagation.
No unit clauses.
Branch on x2.
<spawn [0.1.0]>
( 4 ) ^ ( 4 v 3 ) ^ ( -5 v -3 ) ^ ( -5 v -3 ) ^ ( -3 ) ^ ( -5 v -4 ) ^
( 4 v -5 ) ^ ( -4 v -5 v 3 )
</spawn>
<spawn [0.1.1]>
( 4 v 5 ) ^ ( -5 v 4 ) ^ ( 4 v 3 ) ^ ( 4 v 3 ) ^ ( -5 v -3 ) ^ ( 4 ) ^
( 4 v 5 ) ^ ( -5 v -4 ) ^ ( 4 v -5 ) ^ ( -4 v -5 v 3 ) ^ ( 4 ) ^ ( 4 v 3 )
</spawn>
<recv [0.1.0, 0.1.1]>      % waits for ALL children before proceeding
\end{verbatim}
\end{small}

\vspace{0.6em}
\noindent\textbf{Worker 0.1 prompt (after receiving SAT from \emph{both} workers~0.1.0 and~0.1.1).}
\begin{small}
\begin{verbatim}
Your id is: 0.1.
( 2 v 4 ) ^ ... ^ ( -2 v 4 v 3 )
Try unit propagation.
No unit clauses.
Branch on x2.
<spawn [0.1.0]>
...
</spawn>
<spawn [0.1.1]>
...
</spawn>
<recv [0.1.0, 0.1.1]>
{SAT.}
{SAT.}     % both child results received before aggregating
</recv>
\end{verbatim}
\end{small}

\noindent\textbf{Worker 0.1 response (concludes and reports to root).}
\begin{small}
\begin{verbatim}
Conclusion for this subproblem: SAT.
<send [0]>{SAT.}</send> <stop>
\end{verbatim}
\end{small}

\section{Sample 3-SAT Chain-of-Thoughts (Serial)}
\label{app:sample_cot_sat_serial}

The serial DPLL solver branches depth-first on a single thread with no message
passing.

\vspace{0.5em}
\subsection{Solver Trace}

\noindent\textbf{Prompt.}
\begin{small}
\begin{verbatim}
( 2 v -1 v 4 ) ^ ( 4 v 5 v -2 ) ^ ( -2 v 3 v 1 ) ^ ( -5 v -2 v 4 ) ^
( 4 v 3 v -1 ) ^ ( -5 v 1 v -4 ) ^ ( 4 v -2 v 3 ) ^ ( -5 v 2 v -3 ) ^
( -5 v 2 v 1 ) ^ ( -5 v -1 v -3 ) ^ ( 4 v -2 v -1 ) ^ ( 4 v 5 v -2 ) ^
( 1 v 3 v 2 ) ^ ( -2 v 4 v 1 ) ^ ( 1 v -5 v -4 ) ^ ( 2 v -1 v -3 ) ^
( -1 v -5 v -4 ) ^ ( -1 v 4 v -5 ) ^ ( -4 v -5 v 3 ) ^ ( -1 v 4 v -2 ) ^
( -2 v 4 v 3 )
\end{verbatim}
\end{small}

\noindent\textbf{Response.}
\begin{small}
\begin{verbatim}
Branch on x1.
Branch x1=0
CNF:
( 4 v 5 v -2 ) ^ ( -2 v 3 ) ^ ( -5 v -2 v 4 ) ^ ( -5 v -4 ) ^
( 4 v -2 v 3 ) ^ ( -5 v 2 v -3 ) ^ ( -5 v 2 ) ^ ( 4 v 5 v -2 ) ^
( 3 v 2 ) ^ ( -2 v 4 ) ^ ( -5 v -4 ) ^ ( -4 v -5 v 3 ) ^ ( -2 v 4 v 3 )
  Branch on x2.
  Branch x2=0
  CNF:
  ( -5 v -4 ) ^ ( -5 v -3 ) ^ ( -5 ) ^ ( 3 ) ^ ( -5 v -4 ) ^ ( -4 v -5 v 3 )
    Unit assignments: x5=False, x3=True
    CNF:
    TRUE
    Result: SAT.
<stop>SAT.</stop>
\end{verbatim}
\end{small}

\newpage
\section{Extended Theoretical Results}
\label{sec:extended_theory}

To understand the benefits of our framework, we consider the maximum number of context tokens, and the total number of generated tokens as the metrics of importance, and use it to understand the asymptotic behavior of the MPLM framework and show how it outperforms the FJ framework. Assume that we are trying to solve an $N^2\times N^2$ sudoku. We analyze the CoT of both the parent and children threads for our analysis.

\subsection{Proof of Theorem~\ref{thm:max_context}}
\begin{proof}
We bound the context length required by each type of thread under all execution models.

\paragraph{Serial CoT}
In the serial model, a single LLM thread performs all computation iteratively in context. Therefore, in one iteration, it generates at most $W$ tokens and then an at most $|M|$ token long message for each of the $N$ workers. Therefore, summing over all $T$ iterations, the context consumed is:

\[
C_{\max}^{\mathrm{Serial}} = \bigO{(TN(W + |M|))}.
\]
Since from Assumption \ref{assumption:comm_dominated}, $W = \bigO{(k|M|)}$, we have 

\[
\boxed{C_{\max}^{\mathrm{Serial}} = \bigO{(TNk|M|)}}
\]

\paragraph{Fork-join.}
In the fork-join (FJ) model, worker threads are respawned at every iteration, so their contexts do not persist across iterations. In a single iteration, a worker consumes its prompt and produces local work plus its outgoing message, for a total of at most $W + |M|$ tokens. Thus, the maximum worker context is
\[
C_{\text{worker}}^{\mathrm{FJ}} = \bigO{(W + |M|)}.
\]

The master thread, however, persists across all iterations. In each iteration, it receives one message of size $|M|$ from each of the $N$ ranks, resulting in $O(N|M|)$ new tokens per iteration. Over $T$ iterations, the master accumulates
\[
C_{\text{master}}^{\mathrm{FJ}} = \bigO{(TN|M|)}.
\]

Taking the maximum over all threads yields
\[
C_{\max}^{\mathrm{FJ}} = \bigO{(\max\{\,W + |M|,\; TN|M|\,\})}.
\]
Since, from Assumption \ref{assumption:comm_dominated}, we know that $W = \bigO{(k|M|)}, W+|M| = \bigO{(k|M|)}$. Therefore,

\[
\boxed{C_{\max}^{\mathrm{FJ}} = \bigO{(TN|M|)}}
\]

\paragraph{MPLM-style persistence.}
In the MPLM model, each rank maintains a persistent context across iterations. Fix any rank $i$. In each iteration, rank $i$:
\begin{enumerate}
    \item performs local work contributing $\bigO{(W)}$ tokens;
    \item receive messages of size $|M|$ with at most $k$ neighbors, contributing $\bigO{(k|M|)}$ tokens;
    \item sending a $|M|$ token long message which would require $\bigO{(|M|+k)}$ control tokens for the message and communication bookkeeping.
\end{enumerate}
Thus, the per-iteration context growth is
\[
\bigO{(W + k + |M| + k|M|)} = \bigO{(W + k|M|)},
\]
where lower-order control terms are absorbed.

Since the context persists for $T$ iterations, the total context of rank $i$ is
\[
C_i^{\mathrm{MPLM}} = \bigO{(T(W + k|M|))} = \bigO{(TW + Tk|M|)}.
\]
Taking the maximum over all ranks yields
\[
C_{\max}^{\mathrm{MPLM}} = \bigO{(TW + Tk|M|)},
\]
From Assumption \ref{assumption:comm_dominated}, we know that $W = \bigO({k|M|})$, therefore, 
\[
\boxed{C_{\max}^{\mathrm{MPLM}} = \bigO{(Tk|M|)}},
\]

\end{proof}
\subsection{Sudoku-specific analysis}

In the following section, we derive asymptotic rates for the maximum context required ($C$), total tokens generated ($G$) and the number of sequential tokens ($S$) for the Fork-Join paradigm and MPLMs. 

\label{app:sudoku_theory}
\subsubsection{Analysis of FJ}

\textbf{Parent:} The Chain-of-Thought for the parent is of the following form:
\[
[Initial\;sudoku]\textcolor{blue}{[spawn_1]}[message_1]\textcolor{blue}{[spawn_2]}[message_2]....\textcolor{blue}{[spawn_T]}[message_T]\textcolor{blue}{[Final\;sudoku]}
\]

where T is the number of iterations required to complete the Sudoku using the naked-singles strategy. Note that the text in blue refers to the tokens \textit{generated} by the model. Now, consider the maximum context ($C$) of the parent:

 \[
 C_{parent} = 
 2C_{sudoku} + T\cdot C_{spawn} + T \cdot C_{message}
 \]
 where $C_{sudoku}, C_{spawn} $ and $C_{message}$ are the number of tokens consumed by these parts of the CoT. Since the Sudoku has $N^4$ entries, it consumes $\bigO(N^4)$ tokens. Since spawning requires sending the partially solved Sudoku to all the workers, copying it into the CoT costs $\bigO(N^4)$ tokens. The cost of the message is proportional to the number of live workers which is atmost $\bigO(N^4)$, therefore, $C_{message}$ is also $\bigO(N^4)$.
 \[
 \Rightarrow C_{parent} = \bigO(N^4) + \bigO(T\cdot N^4) + \bigO(T\cdot N^4)
 \]
 \[
 \Rightarrow\boxed{ C_{parent} = \bigO(T\cdot N^4)}
 \]

We also compute the number of tokens generated ($G$) by the model which is:

\[
G_{parent} = T\cdot C_{spawn} + C_{sudoku}
\]
\[
\Rightarrow G_{parent} = \bigO{(T\cdot N^4)} + \bigO{(N^4)}
\]
\[
\Rightarrow \boxed{G_{parent} = \bigO{(T\cdot N^4)}}
\]
 
\textbf{Child:} The CoT for the child is of the following form:
\[
[Current\;sudoku]\textcolor{blue}{[Response]}
\]
where the Response is if the corresponding cell can be filled or not. For the child, 

\[
C_{child} = C_{sudoku} + C_{response}
\]
Computing the response requires listing the possible values left for the cell which is $\bigO(N^2)$. Therefore, $C_{response}$ is $\bigO(N^2)$.
\[
\Rightarrow C_{child} = \bigO(N^4) + \bigO(N^2)
\]
\[
\Rightarrow \boxed{C_{child} = \bigO(N^4)}
\]

Now, consider the total number of tokens generated by \textit{all} the workers. Say that the $i^{th}$ worker runs for $T_i$ iterations, then:
\[
G_{child_i} = \bigO{(T_i\cdot N^2)}
\]
Therefore, combining the tokens from all the workers gives:
\[
\boxed{G_{child} =  \bigO{(T_{total} \cdot N^2)} }
\]
where $T_{total} = \sum_{i=1}^{N^4}{T_i}$

Using the above analysis, its evident that in FJ, the context is dominated by the parent which is $\bigO(T\cdot N^4)$. Therefore, we obtain:

\begin{equation}
\label{eq:fj_context}
\boxed{C^{FJ} = \bigO{(T\cdot N^4)}}
\end{equation}

\begin{equation}
\label{eq:fj_total}
\boxed{G^{FJ} = \bigO{(T\cdot N^4+ T_{total}\cdot N^2)} }    
\end{equation}

\subsubsection{Analysis of MPLM}

\textbf{Parent:} The CoT for the parent in MPLM is of the following form:
\[
[Initial\;sudoku]\textcolor{blue}{[spawn]}[message]\textcolor{blue}{[Final\;sudoku]}
\]
Therefore, 
\[
C_{parent} = 2\cdot C_{sudoku} + C_{spawn} + C_{message}
\]
The cost of the Sudoku is still $\bigO(N^4)$. The cost of spawning the workers and message is also $\bigO(N^4)$.
\[
\Rightarrow C_{parent} = 2\cdot \bigO(N^4) + \bigO(N^4) + \bigO(N^4)
\]
\[
\Rightarrow \boxed{C_{parent} = \bigO(N^4)}
\]

We also compute the total tokens generated by the master which is 
\[
G_{parent} = C_{spawn}+C_{sudoku}
\]
\[
\Rightarrow \boxed{G_{parent} = \bigO{(N^4)}}
\]

\textbf{Child:} Since we are interested in the maximum context consumed, consider the child which runs for the maximum number of iterations. The CoT for the child in MPLM is of the form:

\[
[Initial\;sudoku]\textcolor{blue}{[send_1]}[message_2]\textcolor{blue}{[send_2]}[message_2]...\textcolor{blue}{[send_T]}
\]
Therefore, 
\[
C_{child} = C_{sudoku} + (T-1)\cdot C_{message} + T\cdot C_{send}
\]
\textbf{The cost of the message in MPLM is $\bigO(N^2)$ }since each worker only receives messages from its neighbors, that is, workers in its rows, columns, and grid. Similary, the cost of sending a message is also $\bigO(N^2)$. Therefore, 
\[
C_{child} = \bigO(N^4) + \bigO((T-1)\cdot N^2) + \bigO(T\cdot N^2)
\]
\[
\Rightarrow \boxed{C_{child} = \bigO(N^4 + T\cdot N^2)}
\]
Now, we calculate the number of tokens generated by the $i^{th}$ worker. 
\[
G_{child_i} = T_i \cdot C_{send}
\]
\[
\Rightarrow G_{child_i} = \bigO{(T_i\cdot N^2)}
\]
Therefore, summming over all children, 
\[
\Rightarrow \boxed{G_{child} = \bigO{(T_{total} \cdot N^2)}}
\]
where $T_{total} = \sum_{i=1}^{N^4}{T_i}$.

Therefore, for MPLM we have
\begin{equation}
\label{eq:MPLM_context}
\boxed{C^{MPLM} = \bigO{(N^4 + T\cdot N^2)}}
\end{equation}

\begin{equation}    
\label{eq:MPLM_total}
\boxed{G^{MPLM} = \bigO{(N^4+T_{total}\cdot N^2}}
\end{equation}

Using Equations \Crefrange{eq:fj_context}{eq:MPLM_total} we observe that MPLM is better both in terms of maximum context required and the total number of tokens generated.



\subsubsection{Sequential tokens}

Number of sequential tokens represents the maximum causally dependent tokens. We measure sequential tokens as the sum, over iterations, of the maximum number of
newly generated tokens among all agents in that iteration. Let $G_{i,t}$ denote
the number of generated tokens produced by agent $i$ at iteration $t$ (including the master and workers),

and let $T$ be the number of iterations required to complete the Sudoku using the naked-singles strategy. We define
\[
S \;=\; \sum_{t=0}^{T} \max_i G_{i,t}.
\]

\paragraph{FJ}
In FJ, the master participates every iteration, it generates a \texttt{spawn} containing the
(partially solved) Sudoku and aggregates worker messages.
Per iteration, the master generates $\bigO(N^4)$ tokens for \texttt{spawn},
and the workers computes the response by listing the possible values left for the cell, which is $\bigO(N^2)$ tokens. So
\[
\max_i G_{i,t} =
\begin{cases}
\bigO(N^4+N^2), & 0 \le t \le T-1,\\
\bigO(N^4), & t = T.
\end{cases}
\]
Thus,
\[
S_{\text{FJ}}
= \bigO(T\cdot(N^4+N^2)) + \bigO(N^4)
\]
\begin{equation}
\label{eq:fj_seq}
\boxed{S_{\text{FJ}} = \bigO\!\left(TN^4\right)}.    
\end{equation}

\paragraph{MPLM.}
In MPLM, the master generates a single initial \texttt{spawn} at $t=0$,
workers exchange only local neighbor information during intermediate iterations,
and the master outputs the final solved Sudoku at termination. So,

\[
\max_i G_{i,t} =
\begin{cases}
\bigO(N^4), & t = 0,\\
\bigO(N^2), & 1 \le t \le T-1,\\
\bigO(N^4), & t = T.
\end{cases}
\]
Therefore if $S$ denotes number of sequential tokens, then:
\[
S_{\text{MPLM}}
= \bigO(N^4) + (T-1)\bigO(N^2) + \bigO(N^4)
\]
\begin{equation}   
\label{eq:mplm_seq}
\Rightarrow \boxed{S_{\text{MPLM}} = \bigO(N^4 + T N^2)}
\end{equation}

From \Crefrange{eq:fj_seq}{eq:mplm_seq}, we observe that for the number of sequential tokens as well, MPLMs are asymptotically better.

\section{Respawning Improves Efficiency and Enables Unbounded Execution}
\label{app:respawn}

In our Sudoku setting, a worker is responsible for a single cell and its entire logical state can be summarized by the set of remaining feasible values for that cell. Therefore, after any number of iterations, the state of the worker can be summarized concisely and the worker can be respawned with a prompt containing the set of remaining possible values.

This mechanism has two key effects. First, it prevents unbounded growth of per-worker context: as shown in Figure \ref{fig:input_tokens}, respawning keeps the maximum worker context substantially lower across problem sizes, with the gap widening as Sudoku size increases. Second, it reduces per-iteration input tokens by eliminating historical reasoning that is no longer necessary, leading to the characteristic `sawtooth' pattern where input tokens reset after each respawn.

The reduced size of the prompt also has an effect on the latency of inference since per-token cost increases with the size of the prompt. To test this, we train models for $N=4$ and $N=5$ without and with respawning every 10 iterations and measure the effects on latency.  We observe that respawning yields 11.3\% and 36.0\% reductions in latency for $N=4$ and $N=5$, respectively.

\begin{figure}[t]
    \centering
    \includegraphics[width=0.8\columnwidth]{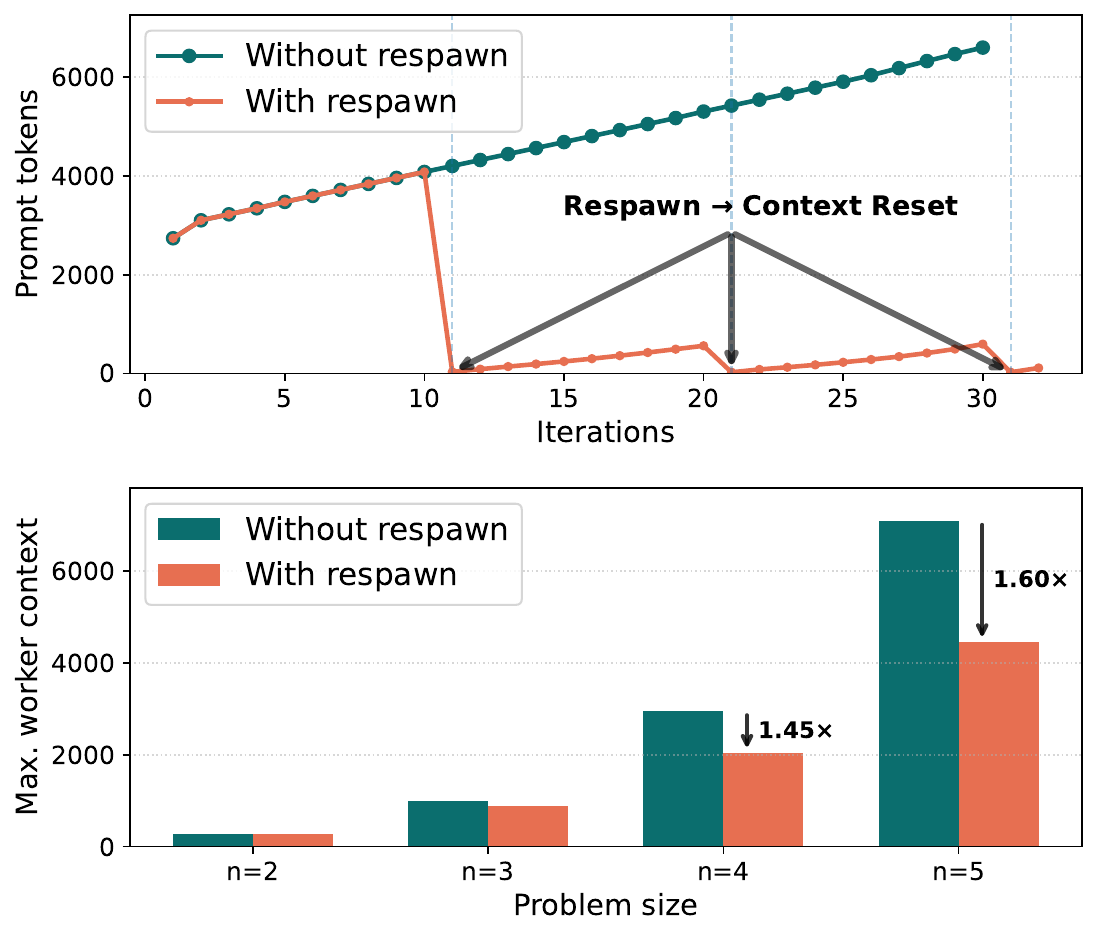}
    \caption{Respawning at a regular frequency (10 iterations in our setting) effectively reduces maximum worker context.}
    \vspace{-1em}
    \label{fig:input_tokens}
\end{figure}

\section{System Profiling and Message-Token Overhead}
\label{app:profiling}

We profile MPLM on representative Sudoku and SAT instances to understand where runtime is spent and how much communication contributes to generation cost. Across all problem sizes, the dominant cost is model execution rather than controller logic. For Sudoku, vLLM accounts for $98.2\%$, $99.1\%$, and $99.6\%$ of wall-clock time on $9\times 9$, $16\times 16$, and $25\times 25$ instances, respectively, while controller overhead remains below $2\%$ and decreases as the problem size grows. This indicates that the message-passing controller adds negligible overhead relative to LLM inference, even as the number of rounds and requests increases substantially. GPU utilization increases with Sudoku size, rising from $68.9\%$ on $9\times 9$ to $90.2\%$ on $25\times 25$, suggesting that larger MPLM workloads better saturate the inference backend. 

\begin{table}[!t]
\centering
\small
\begin{tabular}{lccc}
\toprule
\textbf{Metric} & \textbf{$n=3$ $(9\times 9)$} & \textbf{$n=4$ $(16\times 16)$} & \textbf{$n=5$ $(25\times 25)$} \\
\midrule
Total wall-clock time & 16.99s & 112.898s & 999.014s \\
vLLM execution time & 16.68s & 111.832s & 995.208s \\
vLLM execution ratio & 98.2\% & 99.1\% & 99.6\% \\
Controller overhead & 0.303s & 1.066s & 3.806s \\
Controller overhead ratio & 1.8\% & 0.9\% & 0.4\% \\
Number of rounds & 24 & 46 & 102 \\
Total vLLM requests & 420 & 1,682 & 5,609 \\
Average requests per round & 17.5 & 36.6 & 55.0 \\
\bottomrule
\end{tabular}
\caption{Runtime profiling for MPLM on representative Sudoku instances.}
\label{tab:sudoku-runtime-profiling}
\end{table}

\begin{table}[!t]
\centering
\small
\begin{tabular}{lccc}
\toprule
\textbf{Metric} & \textbf{$n=3$ $(9\times 9)$} & \textbf{$n=4$ $(16\times 16)$} & \textbf{$n=5$ $(25\times 25)$} \\
\midrule
Mean GPU utilization & 68.9\% & 82.5\% & 90.2\% \\
Maximum GPU utilization & 82\% & 93\% & 100\% \\
Mean KV-cache usage & 0.012\% & 0.124\% & 0.443\% \\
Maximum KV-cache usage & 0.030\% & 0.579\% & 1.000\% \\
Minimum KV-cache usage & 0.000\% & 0.000\% & 0.000\% \\
Prefix-cache hit rate & 75.7\% & 84.0\% & 41.7\% \\
\bottomrule
\end{tabular}
\caption{GPU and KV-cache profiling for MPLM on representative Sudoku instances.}
\label{tab:sudoku-gpu-kv-profiling}
\end{table}
\begin{table}[!t]
\centering
\small
\begin{tabular}{lcc}
\toprule
\textbf{Metric} & \textbf{SAT} & \textbf{UNSAT} \\
\midrule
Total wall-clock time & 41.757s & 54.686s \\
vLLM execution time & 41.726s & 54.676s \\
vLLM execution ratio & 99.9\% & 100.0\% \\
Controller overhead & 0.031s & 0.010s \\
Controller overhead ratio & 0.1\% & 0.0\% \\
Number of rounds & 18 & 45 \\
Total vLLM requests & 45 & 72 \\
Average requests per round & 2.39 & 1.71 \\
\bottomrule
\end{tabular}
\caption{Runtime profiling for MPLM on representative SAT instances.}
\label{tab:sat-runtime-profiling}
\end{table}

\begin{table}[!t]
\centering
\small
\begin{tabular}{lcc}
\toprule
\textbf{Metric} & \textbf{SAT} & \textbf{UNSAT} \\
\midrule
Mean GPU utilization & 83.6\% & 83.5\% \\
Maximum GPU utilization & 90\% & 87\% \\
Mean KV-cache usage & 0.023\% & 0.015\% \\
Maximum KV-cache usage & 0.056\% & 0.028\% \\
Minimum KV-cache usage & 0.000\% & 0.000\% \\
Prefix-cache hit rate & 50.8\% & 50.6\% \\
\bottomrule
\end{tabular}
\caption{GPU and KV-cache profiling for MPLM on representative SAT instances.}
\label{tab:sat-gpu-kv-profiling}
\end{table}

\begin{table}[!t]
\centering
\small
\begin{tabular}{llccc}
\toprule
\textbf{Size} & \textbf{Thread Type} & \textbf{CoT Tokens} & \textbf{Directive Tokens} & \textbf{Message Payload Tokens} \\
\midrule
$n=3$ $(9\times 9)$ & Master & 0\% & 53.4\% & 46.6\% \\
$n=3$ $(9\times 9)$ & Child  & 24.2\% & 67.8\% & 8.0\% \\
$n=4$ $(16\times 16)$ & Master & 0\% & 50.1\% & 49.9\% \\
$n=4$ $(16\times 16)$ & Child  & 71.7\% & 26.7\% & 1.6\% \\
$n=5$ $(25\times 25)$ & Master & 0\% & 48.9\% & 51.1\% \\
$n=5$ $(25\times 25)$ & Child  & 76.0\% & 23.1\% & 0.9\% \\
\bottomrule
\end{tabular}
\caption{Output token composition for master and child threads on Sudoku.}
\label{tab:message-token-overhead}
\end{table}

We observe a similar pattern on SAT. For both SAT and UNSAT instances, nearly all wall-clock time is spent inside vLLM, with controller overhead at $0.1\%$ or lower. The SAT instances require fewer parallel requests per round than Sudoku, reflecting the smaller number of active reasoning threads in these examples, but still maintain high GPU utilization around $83.5\%$.

Finally, we measure the composition of generated output tokens on Sudoku to quantify message-token overhead. Master threads produce no chain-of-thought tokens; their outputs consist only of directives and message payloads, as expected from their role as coordinators. Child threads, in contrast, spend most of their output budget on local reasoning for larger puzzles. In particular, message payload tokens account for only $8.0\%$, $1.6\%$, and $0.9\%$ of child-thread output on $9\times 9$, $16\times 16$, and $25\times 25$ Sudoku, respectively. Thus, while MPLM introduces explicit communication, the token overhead of messages is small for worker threads at larger problem sizes, where most generated tokens are devoted to local reasoning.

\pagebreak

\section{Synchronization Schemes Under Variable Latencies Across Workers}
\label{sec:sync_under_variable_latency}

\begin{figure}[t!]
    \centering
    \includegraphics[width=0.85\textwidth]{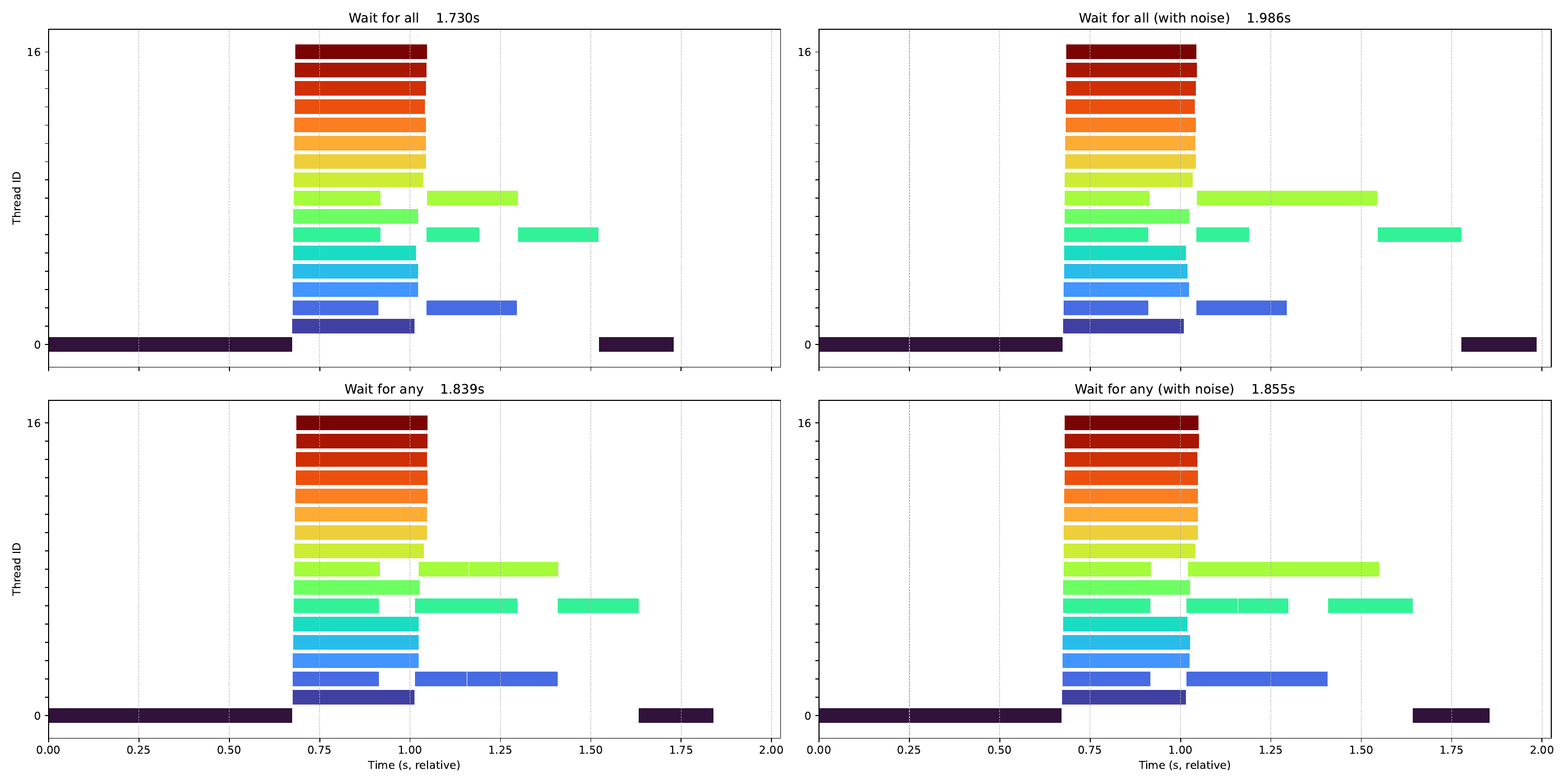}
\caption{In the example above, we note the latency of individual workers in the settings with added noise (right) v.s. no noise (left). The y-axis denotes the ID of the worker and x-axis denotes time in seconds. The start and end of a bar denotes the time frame for which a worker was active during inference.}
    \label{fig:noise_test}
\end{figure}

As we discuss in Section \ref{sec:sync}, there are multiple ways in which ways in which each thread is allowed to invoke the inference engine and \emph{when} messages are delivered. We discuss two implementations: \textbf{wait for all} and \textbf{wait for any}. Technically, they can have different primitives but instead we use the same tokens and change their behavior to understand their individual behaviors without having to re-train the model. 

On our Sudoku instantiation, each worker executes a highly structured and largely deterministic computation. But in practice, worker latencies can vary substantially across iterations and across threads, primarily due to heterogeneous input and output token lengths and system-level contention and queueing on shared GPU resources such as dynamic batching. To simulate this scenario while keeping the underlying reasoning behavior unchanged, we inject an additional post-run delay into each thread to amplify cross-worker and cross-iteration variance. For worker $i$ at iteration $t$, we extend its runtime by
\[
\Delta_{i,t} \;=\; r_{i,t}\cdot u_{i,t}, \qquad u_{i,t}\sim \text{Uniform}(0,1),
\]
where $r_{i,t}$ is the measured thread runtime excluding the injected delay and $u_{i,t}$ is a deterministic pseudo-random draw fixed per $(i,t)$ for reproducibility. This construction increases the latency dispersion within and across iterations without changing the logical content of the messages. We perform an experimental study on the $4\times 4$ Sudoku puzzles where we use the same noise (by fixing the random seed) and use the two synchronization schemes on a dataset of 100 puzzles and then compare the final latencies.

\textbf{Results}: We observe that out of the 100 puzzles, in 76 Sudoku puzzles, \textbf{wait-for-any} is faster because it doesn't have to wait for the slowest worker in its neighbors. However, this can be worse in cases where actually waiting for the slowest worker can be beneficial since more information is gained per iteration which would result in fewer iterations to be used for solving the overall puzzle. Below, we discuss an example where overall latency is lower because of the wait-for-all synchronization scheme. 

\begin{wrapfigure}{r}{0.3\textwidth}
  \centering
  \vspace{-1.5em}
  \includegraphics[width=\linewidth]{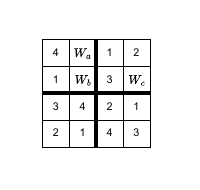}
  \vspace{-4em}
  \caption{An example Sudoku where \textbf{wait-for-any} can be faster than \textbf{wait-for-all} sychronization.}
  \label{fig:wrap}
\end{wrapfigure}
Consider the Sudoku drawn here where there are only 3 cells to be filled denoted by $W_a, W_b$ and $W_c$. Note that $W_a$ and $W_c$  can be resolved after 2 iterations. However, $W_b$ can only be resolved after either $W_a$ or $W_c$ is resolved. Therefore, $W_b$ is critically dependent on $W_a$ and $W_c$ but not both. Even if one of them finishes, $W_b$ can finish its task. Therefore, consider that \textcolor{red}{$W_c$ is much slower}. 

Then, in the \textbf{wait-for-all} setting, $W_b$ will have to wait for $W_c$. Therefore, latency will be increased due to $W_c$'s slowness. 

However, in the \textbf{wait-for-any} setting, $W_b$ can start when $W_a$ is finished. Therefore, there is a possibility of a speedup. 

This result is seen in right part of Figure \ref{fig:noise_test} where \textbf{wait-for-any} is actually faster than \textbf{wait-for-all}.


\section{Long Context Question Answering}
\label{app:long_qa}

\label{sec:long_qa_example}
\begin{figure}[t!]
    \centering
    \includegraphics[width=0.85\textwidth]{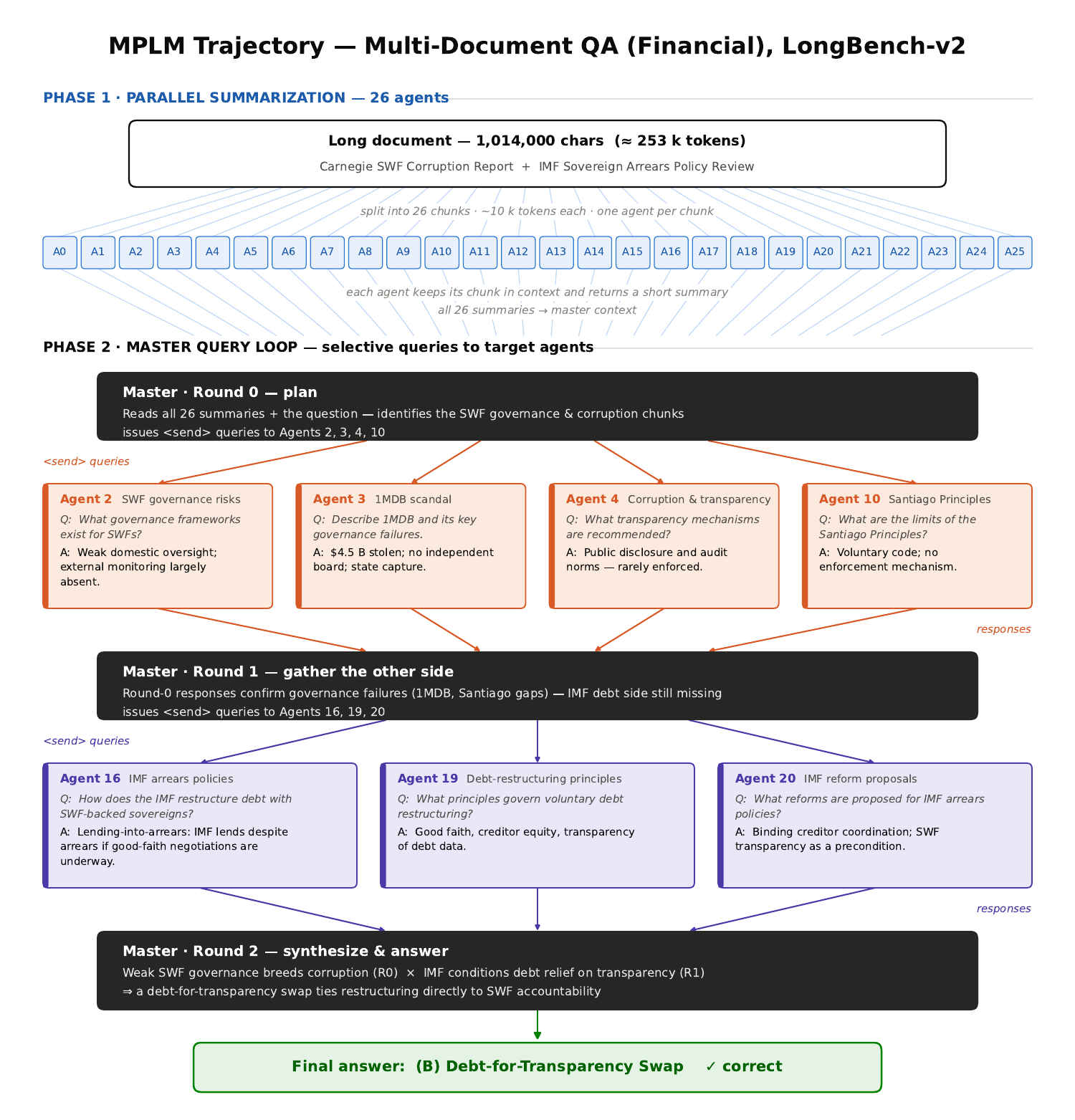}
\caption{In the example above, we show message passing can be used to selective route information across different parallel threads. This figure is an illustrative example from the LongBench-v2 dataset using the Qwen3-30B-A3B model.}
    \label{fig:mplm_example}
    \vspace{-0.5em}
\end{figure}

MPLMs decompose long-context reasoning into two stages: parallel local processing followed by targeted message routing. Given a long input document or collection of documents, the main thread first partitions the context into chunks and spawns a set of parallel worker threads, each assigned to process one chunk. Each worker reads its local chunk, constructs a compact summary, and remains active as a persistent worker associated with that chunk. The main thread receives these summaries, which provide a global index over the document without requiring the full context to be repeatedly copied into every subsequent reasoning step.

After this initial parallel pass, the main thread performs an iterative query-routing loop. Rather than broadcasting every follow-up question to all workers, the main thread uses the chunk summaries and the current reasoning state to identify which persistent workers are relevant. It then sends targeted messages only to those workers, asking them to recover local evidence, answer subquestions, or verify intermediate claims. Since each worker retains its chunk-specific state, the main thread can route new information back to the same worker without retransmitting the original chunk. This enables the system to combine global coordination with local memory: the main thread decides what information is needed next, while persistent workers provide grounded responses from their assigned portions of the context. We report the full results on LongBench-v2 in Table~\ref{tab:fulllongbenchv2}. An example can be seen in Figure \ref{fig:mplm_example}

\begin{table*}[t!]
\centering
\scriptsize
\setlength{\tabcolsep}{3pt}
\resizebox{\textwidth}{!}{
\begin{tabular}{lcccccccc}
\toprule

\multirow{2}{*}{\textbf{Partition}} 
& \multicolumn{4}{c}{\textbf{Qwen3-30B-A3B}} 
& \multicolumn{4}{c}{\textbf{Qwen3.6-35B-A3B}} \\
\cmidrule(lr){2-5} \cmidrule(lr){6-9}
& \makecell{MPLM \\ Accuracy} 
& \makecell{MPLM \\ Latency} 
& \makecell{RLM \\ Accuracy} 
& \makecell{RLM \\ Latency}
& \makecell{MPLM \\ Accuracy}
& \makecell{MPLM \\ Latency}
& \makecell{RLM \\ Accuracy} 
& \makecell{RLM \\ Latency} \\

\midrule

\multicolumn{9}{l}{\textit{I. Single-Document QA}} \\
 Academic & \textbf{39.8\%} & \textbf{52.4s} & 37.5\% & 110.7s & \textbf{47.7\%} & \textbf{88.9s} & \textbf{47.7\%} & 163.0s \\
Literary & \textbf{40.8\%} & \textbf{55.6s} & 32.5\% & 108.2s & \textbf{40.0\%} & \textbf{112.2s} & \textbf{40.0\%} & 210.7s \\
 Legal & \textbf{36.8\%} & \textbf{46.3s} & 34.2\% & 89.8s & 47.4\% & 85.9s & \textbf{68.4\%} & \textbf{84.3s} \\
 Financial & \textbf{63.6\%} & \textbf{57.9s} & 38.6\% & 141.3s & 63.6\% & \textbf{120.8s} & \textbf{68.2\%} & 232.3s \\
 Governmental & \textbf{33.3\%} & \textbf{57.3s} & \textbf{33.3\%} & 119.2s & \textbf{38.9\%} & \textbf{118.9s} & \textbf{38.9\%} & 248.7s \\
Detective & 13.6\% & \textbf{62.9s} & \textbf{46.6\%} & 130.0s & 27.3\% & \textbf{122.2s} & \textbf{36.4\%} & 279.5s \\
Event Ordering & \textbf{35.0\%} & \textbf{82.3s} & 27.5\% & 129.2s & \textbf{75.0\%} & \textbf{123.9s} & 35.0\% & 497.6s \\

\midrule
\multicolumn{9}{l}{\textit{II. Multi-Document QA}} \\
 Academic & \textbf{36.0\%} & \textbf{60.0s} & 17.5\% & 107.1s & \textbf{48.0\%} & \textbf{100.7s} & 42.0\% & 259.6s \\
 Legal & \textbf{30.4\%} & \textbf{61.5s} & 21.4\% & 75.8s & \textbf{50.0\%} & \textbf{95.9s} & 21.4\% & 180.4s \\
 Financial & \textbf{48.3\%} & \textbf{64.1s} & 28.3\% & 114.6s & \textbf{53.3\%} & \textbf{123.1s} & 33.3\% & 279.4s \\
 Governmental & \textbf{35.9\%} & \textbf{57.0s} & 13.0\% & 103.3s & \textbf{34.8\%} & \textbf{121.0s} & \textbf{34.8\%} & 283.5s \\
Multi-News & 26.1\% & \textbf{53.0s} & \textbf{34.8\%} & 112.5s & 52.2\% & \textbf{88.0s} & \textbf{56.5\%} & 139.6s \\

\midrule
\multicolumn{9}{l}{\textit{III. Long In-context Learning}} \\
User Guide QA & \textbf{34.4\%} & \textbf{59.2s} & 25.0\% & 105.4s & \textbf{55.0\%} & \textbf{96.1s} & 40.0\% & 258.7s \\
New Language Translation & \textbf{27.5\%} & \textbf{78.3s} & 22.5\% & 104.9s & \textbf{55.0\%} & \textbf{115.6s} & 50.0\% & 332.9s \\
Many-Shot Learning & \textbf{42.9\%} & \textbf{76.1s} & 33.3\% & 79.2s & \textbf{33.3\%} & \textbf{162.2s} & 23.8\% & 453.6s \\

\midrule
\multicolumn{9}{l}{\textit{IV. Long-dialogue History Understanding}} \\
Agent History QA & \textbf{40.0\%} & \textbf{74.4s} & 28.7\% & 98.0s & \textbf{75.0\%} & 90.0s & 60.0\% & \textbf{70.6s} \\
Dialogue History QA & \textbf{63.2\%} & \textbf{45.0s} & 35.5\% & 87.0s & 47.4\% & 98.4s & \textbf{57.9\%} & \textbf{75.0s} \\

\midrule
\multicolumn{9}{l}{\textit{V. Code Repository Understanding}} \\
Code Repo QA & \textbf{35.0\%} & \textbf{65.5s} & 32.5\% & 93.3s & 28.0\% & \textbf{56.7s} & \textbf{52.0\%} & 187.8s \\

\midrule
\multicolumn{9}{l}{\textit{VI. Long Structured Data Understanding}} \\
Table QA & 37.5\% & \textbf{66.7s} & \textbf{41.7\%} & 99.0s & 27.8\% & \textbf{95.9s} & \textbf{50.0\%} & 96.9s \\
Knowledge Graph Reasoning & \textbf{48.3\%} & \textbf{61.2s} & 11.7\% & 100.8s & 53.3\% & \textbf{112.6s} & \textbf{86.7\%} & 112.9s \\

\midrule
\textbf{Average} 
& \textbf{37.8\%} & \textbf{61.3s} & 29.7\% & 105.7s
& 46.5\% & \textbf{102.2s} & \textbf{46.7\%} & 223.5s \\
\bottomrule
\end{tabular}
}
\caption{Full LongBench-v2 results. Qwen3 evaluations were performed on a single node of 4 GH200 GPUs and the Qwen3.6 evaluations were done on a single node of 4 L40 GPUs.}
\label{tab:fulllongbenchv2}
\vspace{-0.5em}
\end{table*}

\end{document}